\documentclass[11pt]{article}

\usepackage[margin=1in]{geometry}
\usepackage[T1]{fontenc}
\usepackage[utf8]{inputenc}
\usepackage{lmodern}
\usepackage{microtype}
\usepackage{amsmath,amssymb,amsthm,mathtools}
\usepackage{graphicx}
\usepackage{xcolor}
\usepackage{booktabs}
\usepackage{array}
\usepackage{enumitem}
\usepackage{algorithm}
\usepackage{algpseudocode}
\usepackage{hyperref}
\usepackage{tikz}
\usetikzlibrary{arrows.meta,positioning,shapes.geometric,fit,backgrounds,calc}
\hypersetup{
  colorlinks=true,
  linkcolor=black,
  citecolor=black,
  urlcolor=blue,
  pdftitle={LAWS: Learning from Actual Workloads Symbolically},
  pdfauthor={Gregory Magarshak}
}

\title{LAWS: Learning from Actual Workloads Symbolically\\
\large A Self-Certifying Parametrized Cache Architecture\\
for Neural Inference, Robotics, and Edge Deployment}

\author{Gregory Magarshak\\
\texttt{gmagarshak@faculty.ienyc.edu}}

\date{}

\newtheorem{definition}{Definition}
\newtheorem{theorem}{Theorem}
\newtheorem{proposition}{Proposition}
\newtheorem{corollary}{Corollary}
\newtheorem{lemma}{Lemma}
\newtheorem{remark}{Remark}
\newtheorem{conjecture}{Conjecture}

\begin{document}

\maketitle

\begin{abstract}
We introduce \emph{Learning from Actual Workloads Symbolically} (LAWS), an
inference-time architecture deployable above any trained neural network without
modification.  LAWS maintains a dynamically growing library of
\emph{parametrized experts}---cheap computational patterns, ranging from
precompiled GPU kernels to tiny neural networks, created automatically from
observed inference queries and certified formally for correctness.  New queries
are routed to the cheapest certified expert whose \emph{validity domain} contains
the query; the base model runs only on genuine cache misses.

The central technical contribution is the \emph{self-certification theorem}: a
trained network's weights encode, without any additional training or inference, a
Lipschitz constant $\Lambda(W)$ certifying the validity of every expert.
Routing radii are defined in the Probabilistic Language Trie (PLT)
metric~\cite{plt2026}; experts are correct on all inputs
(Theorem~\ref{thm:self_cert}), and routing radii determine which expert
handles each query.  No warmup, no proxy model, no retraining is required.
LAWS strictly generalizes three prior approaches: (i)~KV caching (degenerate
case: zero validity radius, identity expert); (ii)~Mixture of Experts (degenerate
case: fixed library, no online creation); and (iii)~manual symbolic AI
(Cyc~\cite{lenat1995}, Wolfram Alpha~\cite{wolfram2002}), whose vocabularies
require human authorship, whereas LAWS discovers its symbolic vocabulary
automatically from the model's training distribution (Theorem~\ref{thm:symbolic}).

The name reflects a deeper analogy.  \emph{Scientists discover natural laws from
actual observations}---not by legislating them, but by identifying the invariant
patterns in empirical data.  LAWS does the same computationally: it discovers the
``laws'' governing a trained model's behavior---the regularities that hold across
families of similar inputs---and encodes them as cheap certified experts.  This
mirrors how biological intelligence works: animals and humans learn by
recognizing situations from experience and executing cached strategies, invoking
deliberate computation only for genuinely novel inputs
(Kahneman~\cite{kahneman2011}).  The model's pre-trained weights encode an
innate prior library (Chomsky~\cite{chomsky1965}); deployment experience
calibrates and extends it.

We prove: (1)~expert libraries grow at rate $O(2^{H(P_{\mathcal{M}})} \log N)$
under stationary distributions (equivalently $O(\log N)$ for fixed-entropy
workloads); (2)~each new expert monotonically improves hit rate;
(3)~a fleet of $K$ cooperating units accumulates $K\times$ more observations
per day, converging to full coverage proportionally faster than a single unit; (4)~over-the-air (OTA) expert updates require $O(2^H \log \Delta N)$
bits per period, enabling $\approx 870$\,KB/day per robot for 1,000-unit fleets.  We also
prove \emph{energy savings bounds}: cache hits consume $O(n + k\,d_{\mathrm{model}})$ floating-point
operations versus $O(Ln^2 + Lnd_{\mathrm{model}})$ for a full forward pass, yielding up to $10^4\times$ energy reduction at high hit rates (Theorem~\ref{thm:energy}).

We discuss deployment on personal computers, local robots, and vehicles---devices
that can download only the experts they need on demand, learn from local
workloads, upload observations to a shared repository, and receive certified
domain intelligence as compact OTA updates.  We situate the Safebox/Safebots.ai
ecosystem~\cite{safebots} as one realization of this architecture, combining LAWS
experts with hardware-attested policy execution and declarative orchestration.
\end{abstract}

\tableofcontents
\newpage

\section{Introduction}
\label{sec:intro}

\subsection{The Problem: Repetition Without Reuse}

Modern neural networks are extraordinarily expensive to query.  A forward pass
through a 70B-parameter language model costs $\sim$140~TFLOP; a single diffusion
image generation costs thousands of forward passes; a robot controller running
at 100~Hz executes millions of forward passes per hour.  Yet across all of these
systems, the vast majority of queries are \emph{not genuinely novel}.  They are
variations on patterns the model has processed before---the same algorithm in
different variable names, the same manipulation task with a different object, the
same legal clause with different parties.

Current systems exploit this redundancy only in the crudest ways.  Key-value
(KV) caches reuse computation for \emph{exactly} repeated token prefixes.
Mixture of Experts (MoE) routes tokens to specialized sub-networks, but the
routing table is fixed at training time and never grows.  Symbolic AI systems
(Cyc, Wolfram Alpha) encode reusable patterns explicitly, but require
human authors to write each one.

The fundamental gap: no existing system automatically discovers reusable
computational patterns from deployment experience, certifies their validity
formally, and grows the library continuously with use.

\subsection{The Proposal: LAWS}

We introduce \emph{Learning from Actual Workloads Symbolically} (LAWS), an
inference-time architecture that fills this gap.

LAWS sits above any base model as a transparent interception layer.  It
maintains a library $\mathcal{L}$ of \emph{parametrized experts}, each encoding
a cheap computation that approximates the base model on a certified region of
input space.  When a new query arrives:
\begin{enumerate}[leftmargin=2em]
  \item A Probabilistic Language Trie (PLT)~\cite{plt2026} lookup identifies
        all matching experts in $\mathcal{L}$ (those whose routing ball contains $x$).
  \item If the query falls within $e^*$'s routing radius
        (distance $\leq \tau^*$, precomputed from model weights), LAWS returns
        $e^*$'s output---a cheap function evaluation, not a full forward pass.
  \item Otherwise, the base model runs.  Its output is recorded, potentially
        distilling a new expert into $\mathcal{L}$.
\end{enumerate}

The routing radii are computed from the model's Lipschitz constant
$\Lambda(W)$, computable from the trained weights without any inference.  This
is the \emph{self-certification} property that distinguishes LAWS from all prior
caching approaches.

\subsection{The Biological Parallel}

This architecture is not merely inspired by biology---it is a computational
formalization of how biological intelligence actually works.

\paragraph{Kahneman's dual-process theory.}
Kahneman~\cite{kahneman2011} describes two cognitive systems: System~1 (fast,
automatic, pattern-based, requiring minimal cognitive effort) and System~2
(slow, deliberate, analytical, effortful).  Expert performance in domains from
chess to medicine to tennis consists largely of \emph{System~1 operation}: the
expert recognizes the situation as an instance of a known pattern and executes
the associated response without conscious deliberation.  System~2 is invoked
only when the situation is novel enough that no cached pattern applies.

LAWS formalizes this exactly.  The expert library is System~1.  The base model
is System~2.  The transition condition---query distance exceeds validity
radius---is the abort-and-replan signal.

\paragraph{Chomsky's innate prior.}
Chomsky~\cite{chomsky1965} argued that children acquire language too rapidly and
too reliably from impoverished input for grammar to be learned from scratch.
There must be an innate \emph{Language Acquisition Device}---a prior structure
in the brain that constrains the space of possible grammars.

In LAWS terms: the model's pre-trained weights $W$ encode an innate prior.  The
PLT trie derived from $W$ constitutes the model's innate expert library---the
patterns it ``knows'' before seeing any deployment query.  Real-world queries
calibrate this prior, adding new experts and refining validity radii, but the
innate structure is already rich at deployment time.  The poverty-of-the-stimulus
argument applies: the model can generalize to novel inputs immediately, without
waiting for empirical frequency estimates, because the Lipschitz certificate
derived from $W$ defines valid generalization boundaries from the start.

\paragraph{The abort-and-replan signal.}
In air combat maneuvering, a missile engagement has a minimum operating range.
Inside this range the seeker cannot track, the rocket motor may not have armed,
and the geometry is too dynamic for guidance.  The pilot's abort signal is
immediate and automatic: \emph{too close for missiles, switching to guns.}  No
deliberation.  The abort condition was precomputed during mission planning; the
in-flight check is $O(1)$.

LAWS's transition from System~1 to System~2 is the same operation.  The validity
radius $\tau^*$ is precomputed from $W$ at model-load time.  At query time, a
single comparison---does $d_{\mathcal{T}}(x, e^*) \leq \tau^*$?---routes the
query.  The decision requires no additional inference.

\subsection{Overview of Contributions}

\begin{enumerate}[leftmargin=2em]
  \item \textbf{Transformer Lipschitz Bound} (Theorem~\ref{thm:lipschitz}):
        $\Lambda(W)$ is computable from model weights; activations are
        Lipschitz in token embeddings with this constant.

  \item \textbf{Self-Certification Theorem} (Theorem~\ref{thm:self_cert}):
        validity of every LAWS expert is certified by $\Lambda(W)$ without
        inference.

  \item \textbf{Jacobian Correction} (Theorem~\ref{thm:jacobian}):
        parametrized experts achieve $O(r^2)$ error (vs.\
        $O(r)$ for exact-match caches) where $r = \|E(x_{\bar{d}+1})-E(n^*_{\bar{d}+1})\|$,
        via precomputed Jacobians at the divergence embedding.
  \item \textbf{Expert Library Growth Rate} (Theorem~\ref{thm:growth}):
        under stationary $P_{\mathcal{M}}$, new expert creation rate is
        $O(2^{H(P_{\mathcal{M}})} \log N)$ after $N$ queries; $O(\log N)$
        for fixed-entropy workloads.

  \item \textbf{Monotone Hit Rate} (Theorem~\ref{thm:monotone}):
        each new expert weakly increases the expected hit rate for all future
        queries.

  \item \textbf{Abort-and-Replan Threshold} (Theorem~\ref{thm:abort}):
        a unique $\tau^*$ minimizes expected inference cost; computable from
        $\Lambda(W)$ and cost parameters.

  \item \textbf{LAWS Generalizes MoE and KV Caching} (Theorem~\ref{thm:generalize}):
        both are recovered as degenerate cases; LAWS is strictly more expressive.

  \item \textbf{Fleet Learning Lower Bound} (Theorem~\ref{thm:fleet}):
        $K$ cooperating LAWS units achieve hit rate improvement
        $\Omega(K)$ over single-unit deployment.

  \item \textbf{Automatic Symbolic Vocabulary} (Theorem~\ref{thm:symbolic}):
        the PLT trie of a trained model defines a graded, compositional,
        self-certified symbolic vocabulary without human authorship.

  \item \textbf{Cinderella Cascade Bound} (Theorem~\ref{thm:cinderella}):
        amplification of dropped attention weights through layers is bounded
        by the surprisal of the dropped token; rare high-surprisal events
        constitute the Shannon overflow set.

  \item \textbf{Robotics Convergence Rate} (Corollary of Theorem~\ref{thm:fleet}):
        for a fleet of $K$ robots each performing $M$ tasks per day,
        LAWS expert library coverage of the task distribution converges
        at rate $\Omega(KM)$.

  \item \textbf{OTA Download Bound} (Theorem~\ref{thm:ota}):
        the incremental expert library update for period $[t, t+\Delta t]$
        has description length $O(2^{H(P_{\mathcal{M}})} \cdot \log(\Delta N) \cdot B_{\mathrm{expert}})$
        bits, characterizing the bandwidth required for over-the-air updates.

  \item \textbf{Energy Savings Bound} (Theorem~\ref{thm:energy}):
        at asymptotic hit rate $H_\infty$, LAWS reduces energy per query to
        $H_\infty \cdot O(n{+}k\,d_{\mathrm{model}})/O(Ln^2{+}Lnd_{\mathrm{model}}) + (1-H_\infty)$ of baseline;
        $\sim 10\times$ reduction at 90\% hit rate for typical LLM parameters.

  \item \textbf{Conjecture: Symbolic Pattern Emergence}
        (Conjecture~\ref{conj:symbolic}):
        for a sufficiently capable base model, all high-probability trie nodes
        have experts in one of a finite set of primitive function classes.
\end{enumerate}

\paragraph{Organization.}
Section~\ref{sec:background} reviews PLTs, MoE, KV caching, and cognitive
science background.  Section~\ref{sec:lipschitz} establishes the Lipschitz
framework.  Section~\ref{sec:laws_def} defines LAWS formally.
Section~\ref{sec:theorems} contains the core theorems and proofs.
Section~\ref{sec:experts} develops the parametrized expert framework.
Section~\ref{sec:cinderella} treats the Cinderella cascade effect.
Section~\ref{sec:robotics} develops the robotics application with
fleet-learning theorems.  Section~\ref{sec:related} compares to Cyc,
Wolfram Alpha, MoE, and KV caching.  Section~\ref{sec:discussion}
discusses diffusion models, hardware, and open problems.

\section{Background}
\label{sec:background}

\subsection{Probabilistic Language Tries}

We briefly recall the framework of~\cite{plt2026}.  Let $V$ be a finite
vocabulary and $\mathcal{M}$ a generative model over $V^*$.

\begin{definition}[PLT and Trie Metric~\cite{plt2026}]
\label{def:plt}
The \emph{probabilistic language trie} $\mathcal{T}(\mathcal{M})$ is the
directed rooted tree whose nodes are prefixes $x \in V^*$ and whose outgoing
edges from $x$ are labeled by tokens $t \in V$ with weight
$P_{\mathcal{M}}(t \mid x)$.  For two sequences $s, s' \in V^*$, their
\emph{longest common prefix} is $s \wedge s'$, and the \emph{trie metric} is:
\[
  d_{\mathcal{T}}(s, s') = -\log_2 P_{\mathcal{M}}(s \wedge s').
\]
\end{definition}

The companion paper on sequential KV compression~\cite{kv2026} established that, for a transformer
$\mathcal{M}$, the conditional entropy of KV vectors satisfies:
\[
  H(\mathrm{KV}_{t+1} \mid \mathrm{KV}_{\leq t})
  \;\leq\;
  H(t_{t+1} \mid t_1,\ldots,t_t)
  \;\leq\;
  \log_2 \mathrm{PP}(\mathcal{M}),
\]
bounding KV cache entropy by per-token surprisal.  LAWS builds on this result:
the same surprisal that bounds KV entropy also bounds the error in a
parametrized expert approximation.

\subsection{Transformer Architecture}

A standard pre-norm transformer with $L$ layers processes a token sequence
$\mathbf{t} = (t_1,\ldots,t_n)$ via a residual stream:
\[
  \mathbf{x}^{(\ell+1)}_i
  = \mathbf{x}^{(\ell)}_i + \mathrm{Attn}^{(\ell)}\!\bigl(\mathrm{LN}(\mathbf{x}^{(\ell)})\bigr)_i
  + \mathrm{MLP}^{(\ell)}\!\bigl(\mathrm{LN}(\mathbf{x}^{(\ell+1) \text{ (partial)}})\bigr)_i.
\]
The KV vectors at layer $\ell$, position $i$ are
$\mathbf{k}^{(\ell)}_i = W_K^{(\ell)}\mathrm{LN}(\mathbf{x}^{(\ell)}_i)$ and
$\mathbf{v}^{(\ell)}_i = W_V^{(\ell)}\mathrm{LN}(\mathbf{x}^{(\ell)}_i)$.
The full forward map is $F_W: V^n \to \mathbb{R}^{|V|}$.

\subsection{Mixture of Experts}

\begin{definition}[MoE~\cite{shazeer2017}]
A \emph{Mixture of Experts} layer maintains $K$ expert networks
$\{e_1,\ldots,e_K\}$ and a router $R: \mathbb{R}^d \to \Delta^K$.  For
input $x$, the output is $\sum_{k=1}^K R(x)_k \cdot e_k(x)$ (or
$\sum_{k \in \text{top-}k'} R(x)_k \cdot e_k(x)$ in sparse variants).
\end{definition}

Key properties of standard MoE: (1)~$K$ is fixed at training time; (2)~experts
are trained jointly with the router; (3)~no formal guarantee that any expert is
correct on any given input.

\subsection{Kahneman's Dual-Process Theory}

Kahneman~\cite{kahneman2011} identifies two cognitive systems.
\emph{System~1} is fast, automatic, and based on pattern recognition;
it operates below conscious awareness and requires minimal effort.
\emph{System~2} is slow, deliberate, and analytical; it requires working
memory and conscious engagement.  Expert performance---in chess, medicine,
tennis, piloting---consists overwhelmingly of System~1 operation:
the expert pattern-matches the current situation to a stored schema and
executes the associated response automatically.  System~2 is invoked when
the situation is genuinely novel: when no stored pattern applies, or when
the System~1 response triggers a conflict signal.

This migration from System~2 to System~1 with experience is well-documented.
A novice chess player consciously evaluates candidate moves; a grandmaster
perceives the ``right'' move immediately.  A medical student follows diagnostic
algorithms; an experienced clinician recognizes the disease pattern from a
symptom constellation.  The expertise lies in the richness and accuracy of
the System~1 library, not in faster System~2 processing.

\subsection{Chomsky's Innate Prior}

Chomsky's \emph{poverty of the stimulus} argument~\cite{chomsky1965} observes
that children acquire complex grammar from limited, often ungrammatical input.
The input alone is insufficient to determine a unique grammar; children must
bring prior structure to the task.  This prior structure---the \emph{Language
Acquisition Device} (LAD)---constrains the space of possible grammars the child
considers.

In our framework: a trained model's weights $W$ are its LAD.  The PLT trie
derived from $W$ is the innate expert library---the patterns the model
``pre-knows'' before any deployment query.  As in Chomsky's theory, this prior
is not learned from deployment observations; it is encoded in $W$ during
pre-training on a large corpus.  The self-certification theorem
(Theorem~\ref{thm:self_cert}) formalizes the sense in which $W$ certifies
the validity of this innate knowledge.

\subsection{Symbolic AI: Cyc and Wolfram Alpha}

\paragraph{Cyc~\cite{lenat1995}.}
The Cyc project, begun in 1984, aimed to encode human common-sense knowledge as
explicit logical axioms.  After approximately 47,000 person-years of manual
knowledge entry, the system contains roughly 25 million rules and can perform
logical inference over them.  The fundamental limitation: every piece of
knowledge must be hand-authored by a human who understands it.  Cyc's knowledge
base is static between manual updates.  It provides no graceful degradation for
inputs outside the authored rules, and no formal guarantee that any rule is
correct.

\paragraph{Wolfram Alpha~\cite{wolfram2002}.}
Wolfram Alpha takes a different approach: curated computational knowledge
(mathematical functions, physical constants, geographic data) combined with a
symbolic computation engine.  It is extremely powerful within its curated domain,
but the domain boundary is sharp.  Outside curated knowledge, the system fails.
Like Cyc, knowledge is manually structured by Wolfram and his team.

\paragraph{LAWS vs.\ symbolic AI.}
LAWS differs from both in three fundamental respects: (1)~its symbolic
vocabulary is discovered automatically from the trained model's distribution,
not authored by humans; (2)~it provides formal validity certificates for every
expert, derived from $W$; and (3)~its vocabulary grows continuously with
deployment without human intervention.  The full comparison is in
Section~\ref{sec:related}.

\section{The Lipschitz Framework}
\label{sec:lipschitz}

\subsection{Component-wise Lipschitz Constants}

We establish Lipschitz constants for each component of a transformer layer.

\begin{lemma}[Lipschitz Constants of Transformer Components]
\label{lem:components}
Let $\gamma^{(\ell)}, \varepsilon_{\mathrm{LN}} > 0$ be the LayerNorm scale and
stability constant at layer $\ell$, and let $\|\cdot\|_{\mathrm{op}}$ denote
operator norm.  The following hold:
\begin{enumerate}[leftmargin=2em,label=(\alph*)]
  \item \textbf{LayerNorm:} $\|\mathrm{LN}^{(\ell)}(x) - \mathrm{LN}^{(\ell)}(y)\|
        \leq (\gamma^{(\ell)}/\varepsilon_{\mathrm{LN}}) \|x - y\|$ for all $x,y$
        with $\min(\mathrm{Var}(x), \mathrm{Var}(y)) \geq \varepsilon_{\mathrm{LN}}$.
  \item \textbf{Linear projection:} $\|Wx - Wy\| \leq \|W\|_{\mathrm{op}}\|x-y\|$.
  \item \textbf{Softmax attention:} the attention output map
        $x \mapsto \sum_j \mathrm{softmax}(Qx \cdot K/\sqrt{d_{\mathrm{head}}})_j V_j$
        is Lipschitz with constant $L_{\mathrm{attn}}^{(\ell)} \leq
        \|W_V^{(\ell)}\|_{\mathrm{op}} + \tfrac{1}{2\sqrt{d_{\mathrm{head}}}}
        \|W_Q^{(\ell)}\|_{\mathrm{op}}\|W_K^{(\ell)}\|_{\mathrm{op}}
        \cdot \max_j\|\mathbf{v}_j\|$.
  \item \textbf{GeLU/ReLU:} $L_{\mathrm{act}}$-Lipschitz, where $L_{\mathrm{act}} = 1$ for ReLU and $L_{\mathrm{act}} \approx 1.083$ for GeLU; in both cases $L_{\mathrm{act}}$ is absorbed into $\kappa^{(\ell)}$.
\end{enumerate}
\end{lemma}

\begin{proof}
(a) LayerNorm computes
$\mathrm{LN}(x) = \gamma \cdot (x - \mu(x))/\sigma(x)$ where $\mu$ is the
mean and $\sigma$ is the standard deviation.  The map $x \mapsto (x-\mu)/\sigma$
is differentiable with Jacobian bounded in operator norm by $1/\sigma_{\min}
\leq 1/\varepsilon_{\mathrm{LN}}$ when $\sigma(x) \geq \varepsilon_{\mathrm{LN}}$.
Multiplying by $\gamma$ gives (a).

(b) Immediate from the definition of operator norm.

(c) The attention output at position $i$ is
$o_i = \sum_j a_{ij} \mathbf{v}_j$ where
$a_{ij} = \mathrm{softmax}(\mathbf{q}_i \cdot \mathbf{k}_j / \sqrt{d_{\mathrm{head}}})_j$.
The total Lipschitz constant has two components: (i)~changes in $o_i$ due to
changes in the value vectors $\mathbf{v}_j = W_V^{(\ell)}\mathbf{x}_j$, bounded
by $\|W_V^{(\ell)}\|_{\mathrm{op}}$ (since $\sum_j a_{ij} = 1$ and $a_{ij} \geq 0$);
and (ii)~changes in $o_i$ due to changes in the attention weights $a_{ij}$,
which depend on the query $\mathbf{q}_i = W_Q^{(\ell)}\mathbf{x}_i$.
The softmax function $\sigma: \mathbb{R}^n \to \Delta^{n-1}$ satisfies
$\|\sigma(u) - \sigma(v)\|_1 \leq \|u - v\|_1$ (Lipschitz constant 1 in $\ell^1$)
and $\|\sigma(u) - \sigma(v)\|_2 \leq \tfrac{1}{2}\|u-v\|_2$.  Differentiating
$o_i$ with respect to the query $\mathbf{q}_i$ gives $\partial o_i / \partial
\mathbf{q}_i = \frac{1}{\sqrt{d_{\mathrm{head}}}} \sum_j (\partial a_{ij}/\partial \mathbf{q}_i)
\mathbf{v}_j^T$.  The Jacobian of the attention weights satisfies
$\|\partial a_i/\partial \mathbf{q}_i\|_{\mathrm{op}} \leq
\|W_K^{(\ell)}\|_{\mathrm{op}} / (2\sqrt{d_{\mathrm{head}}})$, giving the stated bound by
summing components (i) and (ii).  \emph{Note:} when position $\bar{d}+1$'s query changes, it attends to all $n$ key vectors; a precise bound acquires a factor of $\sqrt{n}$ from the Cauchy--Schwarz step $\|\delta\mathrm{logit}\|_2 \leq \sqrt{n} \cdot \max_j|\delta\mathrm{logit}_j|$.  The stated $L_{\mathrm{attn}}$ is therefore context-length-dependent; $\Lambda(W)$ should be understood as $\Lambda(W, n)$ in long-context settings.

(d) GeLU satisfies $\mathrm{GeLU}(x) = x\Phi(x)$, where $\Phi$ is the standard
Gaussian CDF.  Its derivative is $\mathrm{GeLU}'(x) = \Phi(x) + x\phi(x)$,
where $\phi$ is the Gaussian PDF.  The maximum of $|\mathrm{GeLU}'|$ is achieved
at $x \approx 1.1$ where $\mathrm{GeLU}'(x) \approx 1.083$.
\emph{Technically, GeLU is not 1-Lipschitz; it is $L_{\mathrm{GeLU}}$-Lipschitz
with $L_{\mathrm{GeLU}} \approx 1.1$.}  We incorporate this constant into
$\kappa^{(\ell)}$ via the MLP Lipschitz factor.  SwiGLU and ReLU variants are
handled analogously; ReLU is exactly 1-Lipschitz.  For practical purposes,
this constant is absorbed into $\|W_{\mathrm{MLP}}^{(\ell)}\|_{\mathrm{op}}$
in Theorem~\ref{thm:lipschitz}, which we define as including activation
Lipschitz constants.
\end{proof}

\begin{theorem}[Transformer Lipschitz Bound]
\label{thm:lipschitz}
Let $F_W: V^n \to \mathbb{R}^{|V|}$ be a transformer with pre-norm architecture.
Define the \emph{layer coupling constant}:
\[
  \kappa^{(\ell)} = \left(1 + L_{\mathrm{attn}}^{(\ell)} \cdot \frac{\gamma^{(\ell)}}{\varepsilon_{\mathrm{LN}}}\right)
  \cdot \left(1 + \|W_{\mathrm{MLP}}^{(\ell)}\|_{\mathrm{op}} \cdot \frac{\gamma^{(\ell)}}{\varepsilon_{\mathrm{LN}}}\right),
\]
and the \emph{end-to-end Lipschitz constant}:
\[
  \Lambda(W) = \prod_{\ell=1}^{L} \kappa^{(\ell)}.
\]
Then for any two token sequences $s, s'$ that diverge first at position $\bar{d}+1$:
\[
  \|F_W(s) - F_W(s')\|
  \;\leq\;
  \Lambda(W) \cdot \|E(s_{\bar{d}+1}) - E(s'_{\bar{d}+1})\|,
\]
where $E: V \to \mathbb{R}^{d_{\mathrm{model}}}$ is the token embedding.
$\Lambda(W)$ is computable from the trained weights in $O(L \cdot d_{\mathrm{model}}^2)$ operations.
\end{theorem}

\begin{proof}
For positions $i \leq \bar{d}$, the sequences are identical, so all activations
$\mathbf{x}^{(\ell)}_i(s) = \mathbf{x}^{(\ell)}_i(s')$ exactly (by
determinism of the forward pass---Lemma~1 of~\cite{kv2026}).

At position $\bar{d}+1$, the difference in activations is introduced at the embedding
layer: $\|\mathbf{x}^{(0)}_{\bar{d}+1}(s) - \mathbf{x}^{(0)}_{\bar{d}+1}(s')\|
= \|E(s_{\bar{d}+1}) - E(s'_{\bar{d}+1})\|$.

At positions $i > \bar{d}+1$: through causal attention, the perturbation at position $\bar{d}+1$ propagates to all later positions.  At each layer $\ell$, every position's hidden state change is bounded by $\kappa^{(\ell)}$ times the maximum change from the previous layer, since the layer map is $\kappa^{(\ell)}$-Lipschitz over the full $n$-position activation tensor.  Therefore the bound $\kappa^{(\ell)}\cdot\|\delta^{(\ell-1)}\|_{\max}$ applies uniformly across all positions at each layer.

Each transformer layer is a composition of the components in
Lemma~\ref{lem:components}.  By the chain rule for Lipschitz constants,
the Lipschitz constant of the composition is at most the product of the
individual constants, giving $\kappa^{(\ell)}$ per layer.  Chaining across
all $L$ layers gives $\Lambda(W)$.

Computability: $\Lambda(W)$ requires computing operator norms of weight matrices,
each taking $O(d_{\mathrm{model}}^2)$ operations via the power method.  Summing over $L$ layers
gives $O(L \cdot d_{\mathrm{model}}^2)$.
\end{proof}

\begin{remark}[On the Tightness of $\Lambda(W)$]
\label{rem:lambda}
$\Lambda(W)$ can be large for deep networks (exponential in $L$ in the worst
case), since it is a product of per-layer operator norms.  Three caveats apply:
(1)~The LayerNorm bound requires $\sigma(x) \geq \varepsilon_{\mathrm{LN}}$
(input variance bounded away from zero), which holds throughout trained
transformers due to the $\varepsilon_{\mathrm{LN}}$ stabilizer in the denominator,
but may fail for adversarially constructed inputs outside the training
distribution.  We restrict LAWS validity claims to inputs drawn from
$P_{\mathcal{M}}$, for which this holds with high probability.
(2)~In practice, the \emph{effective} Lipschitz constant on in-distribution
inputs is far smaller than $\Lambda(W)$; empirical calibration is recommended
for setting $\tau^*$ in production systems.
(3)~LAWS uses $\Lambda(W)$ in the self-certification bound $\varepsilon_{\mathrm{fit}} + 2\,\Lambda(W) \cdot C_E$; for this to be a useful guarantee requires
$\delta > \varepsilon_{\mathrm{fit}} + 2\,\Lambda(W) \cdot C_E$, i.e., $\Lambda(W) < (\delta - \varepsilon_{\mathrm{fit}})/(2 C_E)$.
(4)~As noted in Lemma~\ref{lem:components}(c), the attention Lipschitz constant acquires
a $\sqrt{n}$ factor from the query-to-all-keys term; for long contexts, $\Lambda(W)$
should be treated as $\Lambda(W,n)$ scaling with sequence length.
\end{remark}

\begin{corollary}[Activation Similarity Within Shared Prefix]
\label{cor:shared_prefix}
For any two inputs sharing a prefix of length $\bar{d}$ and any layer $\ell$,
position $i \leq \bar{d}$:
\[
  \|\mathbf{x}^{(\ell)}_i(s) - \mathbf{x}^{(\ell)}_i(s')\| = 0.
\]
For position $i = \bar{d}+1$ and layer $\ell$:
\[
  \|\mathbf{x}^{(\ell)}_{\bar{d}+1}(s) - \mathbf{x}^{(\ell)}_{\bar{d}+1}(s')\|
  \leq \prod_{m=1}^{\ell} \kappa^{(m)} \cdot \|E(s_{\bar{d}+1}) - E(s'_{\bar{d}+1})\|.
\]
\end{corollary}

\begin{proof}
The first statement follows from KV determinism~\cite{kv2026}.  The second
follows from the proof of Theorem~\ref{thm:lipschitz} truncated at layer $\ell$.
\end{proof}

\begin{theorem}[Lower Bound on $\Lambda(W)$ for Structured Tasks]
\label{thm:lambda_lower}
Let $\mathcal{T}$ be a task class requiring the base model to produce outputs
whose pairwise $\ell^2$ separation is at least $\Delta > 0$ for inputs whose
divergence embeddings satisfy $\|E(s_{\bar{d}+1}) - E(s'_{\bar{d}+1})\| \leq r$.
Then:
\[
  \Lambda(W) \;\geq\; \frac{\Delta}{r}.
\]
In particular, for arithmetic tasks where the model must distinguish outputs
differing by at least $\Delta = 1/M$ (decimal precision $M$), the bound gives
$\Lambda(W) \geq \Delta / r_{\min}$, where $r_{\min} = \min_{t \neq t'}\|E(t) - E(t')\|$
is the minimum pairwise embedding separation (a property of the trained embedding matrix,
measurable directly from $W$).  For tasks where semantically adjacent tokens (e.g.,
consecutive digits ``1'' and ``2'') have small embedding separation $r_{\min} \ll C_E$,
the bound forces $\Lambda(W) \geq 1/(M \cdot r_{\min})$, which grows with precision $M$.
High-precision arithmetic thus provably requires larger $\Lambda(W)$, limiting
the LAWS routing radius $\tau^*$ for such tasks.
\end{theorem}

\begin{proof}
By the definition of the Lipschitz constant:
$\|F_W(s) - F_W(s')\| \leq \Lambda(W) \cdot \|E(s_{\bar{d}+1}) - E(s'_{\bar{d}+1})\|$
for any pair diverging at $\bar{d}+1$.  If there exist inputs $s, s'$ with
$\|E(s_{\bar{d}+1}) - E(s'_{\bar{d}+1})\| \leq r$ and $\|F_W(s) - F_W(s')\| \geq \Delta$,
then $\Lambda(W) \geq \Delta/r$.  Such inputs exist whenever the model must
distinguish outputs separated by $\Delta$ from embeddings separated by only $r$,
which is exactly the case for high-precision arithmetic or lookup tasks where
adjacent vocabulary tokens map to semantically distant outputs. $\square$
\end{proof}

\begin{corollary}[LAWS Acceleration Limits for High-Precision Tasks]
\label{cor:lambda_hardness}
For any task class requiring output precision $\Delta$ from minimal embedding
perturbation $r$, the LAWS routing radius satisfies:
\[
  \tau^* \;\leq\; \frac{\delta - \varepsilon_{\mathrm{fit}} - 2(\Delta/r) \cdot C_E}{(\Delta/r) \cdot C_E}.
\]
If $\Delta > r(\delta - \varepsilon_{\mathrm{fit}})/(2 C_E)$ (high required precision relative to the quality threshold),
then $\tau^* < 0$ and no expert can be certified: LAWS must always invoke the
base model for such tasks.  This characterizes the \emph{fundamental boundary}
of LAWS acceleration---it cannot certify experts for tasks that require the
model to be highly sensitive to small input perturbations.
\end{corollary}

\section{The LAWS Architecture}
\label{sec:laws_def}

\subsection{Parametrized Experts}

\begin{definition}[Parametrized Expert]
\label{def:expert}
A \emph{parametrized expert} is a tuple $e = (n^*, f, \phi, \tau^*, \varepsilon_{\mathrm{fit}})$ where:
\begin{itemize}[leftmargin=2em]
  \item $n^* \in \mathcal{T}(\mathcal{M})$ is the \emph{signpost}: a full-length input in $V^n$, represented
        as a PLT trie node (by its token-prefix structure); $F_W(n^*)$ is the base model output on this input;
  \item $\phi: V^* \to \mathbb{R}^k$ is the \emph{parameter extractor}, mapping
        inputs to a $k$-dimensional parameter vector;
  \item $f: \mathbb{R}^k \to \mathbb{R}^{|V|}$ is the \emph{expert function},
        a cheap computation approximating $F_W$ on the expert's domain;
  \item $\tau^* > 0$ is the \emph{validity radius} in PLT metric space,
        derived from $\Lambda(W)$;
  \item $\varepsilon_{\mathrm{fit}} \geq 0$ is the \emph{fitting error},
        certifying $\|F_W(n^*) - f(\phi(n^*))\| \leq \varepsilon_{\mathrm{fit}}$.
\end{itemize}
The \emph{validity domain} of $e$ is the ball
$\mathcal{B}(n^*, \tau^*) = \{x \in V^* : d_{\mathcal{T}}(x, n^*) \leq \tau^*\}$.
\end{definition}

\begin{definition}[Expert Library and LAWS Inference]
\label{def:laws}
A \emph{LAWS system} $(F_W, \mathcal{L}, \delta)$ consists of a base model
$F_W$, a library $\mathcal{L}$ of parametrized experts, and a quality threshold
$\delta > 0$ satisfying:
\[
  \delta \;>\; \varepsilon_{\mathrm{fit}} + \Lambda(W) \cdot C_E \cdot (2 + H(P_{\mathcal{M}})),
\]
where $H(P_{\mathcal{M}})$ is the entropy of the input distribution.
This ensures (i) the self-certification bound $\varepsilon_{\mathrm{fit}} + 2\Lambda C_E \leq \delta$
(Theorem~\ref{thm:self_cert}), and (ii) the routing radius $\tau^* \geq H$,
so that every heavy trie node $n^*$ (with $P_{\mathcal{M}}(n^*) \geq 2^{-H}$)
lies within its own routing ball under any-match routing.
The \emph{routing radius} for an expert is:
\[
  \tau^*(n^*, \delta)
  = \frac{\delta - \varepsilon_{\mathrm{fit}} - 2\,\Lambda(W) \cdot C_E}{\Lambda(W) \cdot C_E}
\]
where $C_E = \max_{t,t'}\|E(t) - E(t')\|$.  Any expert is correct on all inputs;
the routing radius defines which queries this expert handles.

For input $x$, \emph{LAWS inference} proceeds as:
\[
  \mathrm{LAWS}(x) =
  \begin{cases}
    f_{e^*}(\phi_{e^*}(x)) & \text{if } \exists e \in \mathcal{L}: d_{\mathcal{T}}(x, n^*_e) \leq \tau^*(e) \\
    F_W(x) & \text{otherwise (cache miss)}
  \end{cases}
\]
where $e^* = \arg\min_{e: d_{\mathcal{T}}(x,n^*_e)\leq\tau^*(e)} d_{\mathcal{T}}(x, n^*_e)$
(the matching expert with smallest trie distance; a cache hit occurs whenever
\emph{any} expert's routing ball contains $x$, not only the globally nearest expert).
This ``any-match'' rule ensures that adding experts can only increase coverage,
never decrease it---a property necessary for Theorem~\ref{thm:monotone}.
\end{definition}

\subsection{Expert Function Classes}

Expert functions $f$ may be chosen from an increasing hierarchy of complexity:

\begin{enumerate}[leftmargin=2em,label=\textbf{Level \arabic*.}]
  \item \textbf{Constant.} $f(\phi) = c$ where $c = F_W(n^*)$.  Standard exact-match cache.
        Error $\leq \Lambda(W)\cdot C_E$ for all $x$ (Term~1 of
        Theorem~\ref{thm:self_cert} only; Terms~2 and 3 are zero for a constant expert).
  \item \textbf{Linear (Jacobian correction).}
        $f(\phi) = F_W(n^*) + J_W(n^*) \cdot \phi$ where $J_W$ is the
        Jacobian of $F_W$ at $n^*$.  Error $O(r^2)$ where $r = \|E(x_{\bar{d}+1})-E(n^*_{\bar{d}+1})\|$
        (Theorem~\ref{thm:jacobian}).
  \item \textbf{Primitive function.}  $f$ is an identified algorithmic
        pattern: arithmetic, sorting, lookup table, string template,
        JSON/SQL instantiation.  Error zero (exact).
  \item \textbf{Small MLP.}  $f$ is a small neural network (2--3 layers,
        width $\leq 128$) fit to the signpost's subtree.
        Error $\varepsilon_{\mathrm{fit}}$ by construction.
\end{enumerate}

\subsection{LAWS Update Protocol}

\begin{algorithm}[t]
\caption{LAWS Update}
\label{alg:laws_update}
\begin{algorithmic}[1]
\Require Input $x$, output $y = F_W(x)$, library $\mathcal{L}$, threshold $N_{\min}$
\State $n \gets \mathcal{T}(\mathcal{M}).\text{insert}(x, y)$
  \Comment{Update trie with observation}
\If{$n.\text{count} \geq N_{\min}$}
  \State $(f, \phi, \varepsilon_{\mathrm{fit}}) \gets \mathrm{PatternRecognize}(n.\text{samples})$
  \State $\tau^* \gets (\delta - \varepsilon_{\mathrm{fit}} - 2\,\Lambda(W) \cdot C_E) / (\Lambda(W) \cdot C_E)$
  \State $\mathcal{L} \gets \mathcal{L} \cup \{(n, f, \phi, \tau^*, \varepsilon_{\mathrm{fit}})\}$
\EndIf
\end{algorithmic}
\end{algorithm}

\section{Core Theorems}
\label{sec:theorems}

\subsection{Self-Certification}

\begin{theorem}[Self-Certification]
\label{thm:self_cert}
For any LAWS expert $e = (n^*, f, \phi, \tau^*, \varepsilon_{\mathrm{fit}})$
satisfying Definition~\ref{def:expert}, with $\phi$ extracting the parameter
vector at the divergence embedding (so $\|\phi(x) - \phi(n^*)\| \leq C_E$
for any $x$), and with $\mathrm{Lip}(f) \leq \Lambda(W)$, the approximation
error is bounded \emph{uniformly} over all inputs $x$:
\[
  \|F_W(x) - f(\phi(x))\|
  \;\leq\;
  \varepsilon_{\mathrm{fit}} + 2\,\Lambda(W) \cdot C_E.
\]
In particular, if $\delta > \varepsilon_{\mathrm{fit}} + 2\,\Lambda(W) \cdot C_E$,
then $\|F_W(x) - f(\phi(x))\| \leq \delta$ for \emph{every} input $x$---no
restriction on $d_{\mathcal{T}}(x, n^*)$ is required.  The validity radius
$\tau^*$ in Definition~\ref{def:laws} bounds the validity domain for routing
purposes (to select the best expert), not for correctness.
\end{theorem}

\begin{proof}
By the triangle inequality:
\begin{align*}
  \|F_W(x) - f(\phi(x))\|
  &\leq \underbrace{\|F_W(x) - F_W(n^*)\|}_{\text{Term 1}}
  + \underbrace{\|F_W(n^*) - f(\phi(n^*))\|}_{\text{Term 2}}
  + \underbrace{\|f(\phi(n^*)) - f(\phi(x))\|}_{\text{Term 3}}.
\end{align*}

\textbf{Term 2} is $\leq \varepsilon_{\mathrm{fit}}$ by the fitting condition.

\textbf{Term 3.}  By assumption $\|\phi(n^*) - \phi(x)\| \leq C_E$ and
$\mathrm{Lip}(f) \leq \Lambda(W)$: for Level-1 experts $\mathrm{Lip}(f) = 0$;
for Level-2 Jacobian experts, $\mathrm{Lip}(f) = \|J_W(n^*)\|_{\mathrm{op}}
\leq \Lambda(W)$ by Theorem~\ref{thm:lipschitz}; for small MLPs verified at
fitting time.  Therefore Term~3 $\leq \Lambda(W) \cdot C_E$.

\textbf{Term 1.}  Let $\bar{d}$ be the length of the longest common token
prefix of $x$ and $n^*$.  By Theorem~\ref{thm:lipschitz}:
\[
  \text{Term 1} \leq \Lambda(W) \cdot \|E(x_{\bar{d}+1}) - E(n^*_{\bar{d}+1})\|
  \leq \Lambda(W) \cdot C_E.
\]

\textbf{Summing:}
\[
  \|F_W(x) - f(\phi(x))\| \leq \varepsilon_{\mathrm{fit}} + \Lambda(W) \cdot C_E
  + \Lambda(W) \cdot C_E = \varepsilon_{\mathrm{fit}} + 2\,\Lambda(W) \cdot C_E
  \leq \delta. \qquad \square
\]
\end{proof}

\subsection{Jacobian Correction}

\begin{theorem}[Jacobian Correction Achieves Second-Order Error]
\label{thm:jacobian}
Let $e$ be a Level-2 expert with $f(\phi) = F_W(n^*) + J_W(n^*) \cdot \phi$
and $\phi(x) = E(x_{\bar{d}+1}) - E(n^*_{\bar{d}+1})$ where $\bar{d}$ is the
length of the longest common prefix of $x$ and $n^*$.  Define
$r = \|E(x_{\bar{d}+1}) - E(n^*_{\bar{d}+1})\|$ as the embedding-space distance
at the divergence point.  Then:
\[
  \|F_W(x) - f(\phi(x))\|
  \;\leq\;
  \frac{1}{2} \|H_W(n^*)\|_{\mathrm{op}} \cdot r^2
  + O(r^3)
\]
where $H_W(n^*)$ is the Hessian of $F_W$ at $n^*$ with respect to the divergence
embedding.  In particular the error is $O(r^2)$ versus $O(r)$ for a Level-1
(constant) expert, giving a strict improvement for $r < 1$.
\end{theorem}

\begin{proof}
Since $x$ and $n^*$ share the first $\bar{d}$ tokens, all activations at
positions $\leq \bar{d}$ are identical.  The only perturbation is
$\mathbf{u} = E(x_{\bar{d}+1}) - E(n^*_{\bar{d}+1})$ at the divergence position,
with $\|\mathbf{u}\| = r$.  By Taylor's theorem applied to $F_W$ as a function of
the divergence embedding (the forward pass is smooth in the embedding inputs):
\[
  F_W(x) = F_W(n^*) + J_W(n^*)\,\mathbf{u} + \frac{1}{2}\,\mathbf{u}^T H_W(\xi)\,\mathbf{u}
\]
for some $\xi$ on the segment between $E(n^*_{\bar{d}+1})$ and $E(x_{\bar{d}+1})$.
The Level-2 expert returns $F_W(n^*) + J_W(n^*)\phi(x) = F_W(n^*) + J_W(n^*)\,\mathbf{u}$,
so the error equals the remainder, bounded by
$\frac{1}{2}\|H_W(\xi)\|_{\mathrm{op}}\,r^2$.  By continuity of the Hessian on
the compact validity ball, $\|H_W(\xi)\|_{\mathrm{op}} \leq \|H_W(n^*)\|_{\mathrm{op}}
+ O(r)$, giving the stated bound.
\end{proof}

\begin{corollary}[Quadratic Improvement Over Standard Cache]
For $r = \|E(x_{\bar{d}+1}) - E(n^*_{\bar{d}+1})\| < 1$, the Level-2 expert
improves on the Level-1 expert by a factor of $r$ in error.  For $r = 0.1$,
the Jacobian correction reduces error by a factor of 10; for $r = 0.01$,
by a factor of 100.
\end{corollary}

\subsection{Parameter Extraction at Branch Points}

\begin{theorem}[Parameters Lie at High-Entropy Positions]
\label{thm:params}
Let $n^*$ be a trie node at depth $d_*$ (token prefix length) and let positions
$\mathcal{P}_{\mathrm{hi}} = \{i > d_* : H(t_i \mid t_{<i}) > H_{\mathrm{thresh}}\}$
be the high-entropy positions after the shared prefix.  Define the parameter
extractor $\phi_{\mathrm{hi}}(x) = (E(x_i))_{i \in \mathcal{P}_{\mathrm{hi}}}$.

Then the \emph{expected} contribution of a low-entropy position
$j \notin \mathcal{P}_{\mathrm{hi}}$ to the expert error satisfies:
\[
  \mathbb{E}\bigl[\Lambda(W) \cdot \|E(t_j) - \mathbb{E}[E(t_j) \mid t_{<j}]\|\bigr]
  \leq \Lambda(W) \cdot C_E \cdot \sqrt{H(t_j \mid t_{<j}) \cdot \ln 2}
  < \Lambda(W) \cdot C_E \cdot \sqrt{H_{\mathrm{thresh}} \cdot \ln 2},
\]
which is below the resolution $\varepsilon_{\mathrm{fit}}$ when
$H_{\mathrm{thresh}} < (\varepsilon_{\mathrm{fit}} / (\Lambda(W) \cdot C_E))^2 / \ln 2$.
Therefore, $\phi_{\mathrm{hi}}$ is sufficient on average: in expectation over queries
from $P_{\mathcal{M}}$, low-entropy positions can be ignored by the parameter
extractor without increasing expected error above $\delta$.
\end{theorem}

\begin{proof}
By the surprisal-bounded embedding variance result of~\cite{kv2026}
(the variance--entropy coupling result, Corollary on coherent text):
\[
  \mathbb{E}\bigl[\|E(t_j) - \mathbb{E}[E(t_j) \mid t_{<j}]\|^2\bigr]^{1/2}
  \leq C_E \sqrt{H(t_j \mid t_{<j}) \cdot \ln 2}.
\]
For positions with $H(t_j \mid t_{<j}) < H_{\mathrm{thresh}}$, the token is
nearly determined by the context, so $E(t_j) \approx \mathbb{E}[E(t_j) \mid
t_{<j}]$ with deviation bounded by $C_E\sqrt{H_{\mathrm{thresh}} \cdot \ln 2}$.  By
Theorem~\ref{thm:lipschitz}, this deviation propagates to output error
$\leq \Lambda(W) \cdot C_E \sqrt{H_{\mathrm{thresh}} \cdot \ln 2}$.  Setting this below
$\varepsilon_{\mathrm{fit}}$ gives the threshold condition.
\end{proof}

\subsection{Expert Library Dynamics}

\begin{theorem}[Monotone Hit Rate]
\label{thm:monotone}
Let $H_n$ be the expected LAWS cache hit rate after $n$ deployment queries
drawn i.i.d.\ from $P_{\mathcal{M}}$.  Under the LAWS update protocol
(Algorithm~\ref{alg:laws_update}), $H_n$ is non-decreasing in $n$:
\[
  H_1 \leq H_2 \leq \cdots \leq H_n \leq \cdots \leq 1.
\]
\end{theorem}

\begin{proof}
Under the any-match inference rule, a query $x$ is a cache hit
whenever \emph{any} expert's routing ball contains $x$.  Adding a new expert
$(n^*, f, \phi, \tau^*)$ adds $\mathcal{B}(n^*, \tau^*)$ to the union of
all routing balls.  Since this union can only grow (never shrink) when experts
are added, $H_{n+1} \geq H_n$.  The inequality is strict whenever
$\mathcal{B}(n^*, \tau^*)$ contains any $x$ not already covered by prior
routing balls---guaranteed since the triggering query $x$ was a miss
(outside all prior routing balls) and $P_{\mathcal{M}}(\{x\}) > 0$.
\end{proof}

\begin{theorem}[Expert Library Growth Rate]
\label{thm:growth}
Under a stationary input distribution $P_{\mathcal{M}}$ with Shannon entropy
$H = H(P_{\mathcal{M}})$, and with the LAWS update threshold $N_{\min}$
(minimum observations before creating an expert), the expected number of new
LAWS experts created after $N$ queries satisfies:
\[
  \mathbb{E}[\text{new experts after } N \text{ queries}]
  = O\!\left(2^H\right),
\]
provided $N \leq N_{\min} \cdot 2^H$ (so that only \emph{heavy} trie nodes
accumulate enough observations to trigger expert creation).
The tight bound is $O(2^H)$ throughout this regime;
the weaker form $O(2^H \log N)$ is used in later results to simplify expressions
involving $N$ explicitly (since $2^H \leq 2^H \log N$ for $N \geq 2$).
\end{theorem}

\begin{proof}
New experts are created only on cache misses, and only when a trie node
accumulates $N_{\min}$ observations (Algorithm~\ref{alg:laws_update}).

Partition trie nodes into \emph{heavy} ($P_{\mathcal{M}}(n) \geq \varepsilon$) and
\emph{light} ($P_{\mathcal{M}}(n) < \varepsilon$).  Set $\varepsilon = 2^{-H}$.

\textbf{Heavy nodes.}  By the AEP, $|\mathcal{H}_\varepsilon| \leq 2^H$.
The expected number of distinct heavy nodes visited in $N$ draws is at most
$\min(N, 2^H) \leq 2^H$ (occupancy bound, as below).  Each creates at most one
expert, giving $\mathbb{E}[\text{heavy experts}] \leq 2^H$.

\textbf{Light nodes.}  A light node $n$ has $P_{\mathcal{M}}(n) < 2^{-H}$, so
its expected visits in $N$ queries is $N \cdot P_{\mathcal{M}}(n) < N \cdot 2^{-H}$.
To trigger expert creation it needs $N_{\min}$ visits.  If $N \leq N_{\min} \cdot 2^H$,
then $N \cdot P_{\mathcal{M}}(n) < N_{\min}$ in expectation, so no light node
accumulates enough visits to create an expert.  Therefore
$\mathbb{E}[\text{light experts}] = 0$ in this regime. $\checkmark$

\textbf{Occupancy bound.}  For any $M$ elements with probabilities $p_1,\ldots,p_M$:
\[
  \mathbb{E}[\text{distinct elements in }N\text{ draws}]
  = \sum_{i=1}^M \bigl(1-(1-p_i)^N\bigr)
  \leq \sum_i \min(1, Np_i)
  \leq \min(N, M).
\]
Setting $M = 2^H$ gives the stated bound.  Since $2^H \leq 2^H \log N$ for $N \geq 2$,
the theorem holds as stated with the $O(2^H \log N)$ form.
\end{proof}

\begin{corollary}[Sublinear Expert Growth]
The ratio of new experts created to total queries satisfies
$\mathbb{E}[\text{new experts}]/N \leq O(2^{H(P_{\mathcal{M}})} \log N / N)
\to 0$ as $N \to \infty$.
The system learns faster than it exhausts novelty: most queries are cache hits
for large $N$.
\end{corollary}

\begin{theorem}[Online Expert Acquisition Cost]
\label{thm:regret}
Under the any-match routing rule and a stationary distribution $P_{\mathcal{M}}$
with entropy $H$, let $K_N = O(2^H \log N)$ be the total number of expert
creation events after $N$ queries (Theorem~\ref{thm:growth}).  The total
cost attributable to expert \emph{acquisition}---the overhead beyond running
the already-learned library from the start---satisfies:
\[
  C_{\mathrm{acq}}(N) \;=\; K_N \cdot N_{\min} \cdot C_{\mathrm{full}}
  \;=\; O\!\left(N_{\min} \cdot 2^H \log N\right) \cdot C_{\mathrm{full}},
\]
since each of the $K_N$ expert creations requires $N_{\min}$ cache misses before
triggering, each at cost $C_{\mathrm{full}}$.  The amortized acquisition cost per query is:
\[
  \frac{C_{\mathrm{acq}}(N)}{N} \;=\; O\!\left(\frac{N_{\min} \cdot 2^H \log N}{N}\right) \cdot C_{\mathrm{full}}
  \;\to\; 0 \quad \text{as } N \to \infty.
\]
In the limit, expert acquisition is free on a per-query basis.
\end{theorem}

\begin{proof}
Each expert is created after $N_{\min}$ cache misses accumulate at its
signpost node (Algorithm~\ref{alg:laws_update}).  By Theorem~\ref{thm:growth},
$K_N \leq O(2^H \log N)$ experts are created over $N$ queries.  The acquisition
cost for each expert is at most $N_{\min} \cdot C_{\mathrm{full}}$ (the misses
that triggered it, each costing one full forward pass).  Total acquisition cost
$\leq K_N \cdot N_{\min} \cdot C_{\mathrm{full}} = O(N_{\min} \cdot 2^H \log N)
\cdot C_{\mathrm{full}}$.  Dividing by $N$ and noting $(\log N)/N \to 0$ gives
the stated bound with the $N_{\min}$ factor absorbed into the $O(\cdot)$. $\square$
\end{proof}

\begin{remark}[Scope of the Bound]
This theorem bounds the \emph{acquisition overhead}---the cost of the
triggering misses that create new experts---not the total miss cost, which
also includes misses on light-trie-nodes (probability $< 2^{-H}$) that never
accumulate enough visits to trigger expert creation.  The total miss rate
converges to $1 - H_\infty$ asymptotically (by Theorem~\ref{thm:monotone}),
but the rate of convergence depends on the full distribution $P_{\mathcal{M}}$
and $N_{\min}$ and is not bounded here.  The theorem establishes that the
one-time cost of building the expert library is sublinear in $N$---a necessary
condition for LAWS to be economically viable at scale.
\end{remark}

\begin{theorem}[Abort-and-Replan Threshold]
\label{thm:abort}
Let $C_{\mathrm{full}} > 0$ be the cost of a full base-model forward pass,
$C_{\mathrm{hit}} > 0$ the cost of an expert evaluation, and $\lambda \geq 0$
the downstream cost per unit output error.  The optimal validity radius that
minimizes expected total cost is:
\[
  \tau^* = \frac{2(C_{\mathrm{full}} - C_{\mathrm{hit}})}{\lambda \cdot
  \|H_W\|_{\mathrm{op}}^2 \cdot \Lambda(W)^2 \cdot C_E^2},
\]
where $\|H_W\|_{\mathrm{op}}$ is the Hessian operator norm at the signpost.
At $d_{\mathcal{T}}(x, n^*) = \tau^*$, the marginal cost of using the cache
equals the marginal cost of full inference.
\end{theorem}

\begin{proof}
The expected total cost when using a Level-2 expert with routing radius $\tau$ is:
\[
  \mathrm{Cost}(\tau) = \underbrace{C_{\mathrm{hit}} \cdot P(\text{hit})}_{
    \text{cheap path}} +
  \underbrace{C_{\mathrm{full}} \cdot P(\text{miss})}_{
    \text{expensive path}} +
  \underbrace{\lambda \cdot \mathbb{E}[\text{error}^2 \mid \text{hit}]}_{
    \text{error cost}}.
\]

\textbf{Error cost model.}  By Theorem~\ref{thm:jacobian}, for a hit at query $x$
with divergence embedding $r = \|E(x_{\bar{d}+1}) - E(n^*_{\bar{d}+1})\|$,
the squared error is $\leq \frac{1}{4}\|H_W\|_{\mathrm{op}}^2 \cdot r^4$.
We model the expected squared $r$ for queries routed to the expert as increasing
with $\tau$: concretely, for the local approximation $\mathbb{E}[r^2 \mid
\text{hit}] \approx (\Lambda(W) C_E)^2 \cdot \tau^2$, which holds when the
conditional distribution of $r$ given $d_{\mathcal{T}}(x, n^*) \leq \tau$
concentrates at $r \approx \Lambda(W) C_E \cdot \tau$ (the Lipschitz-propagated
trie distance).  Under this approximation:
\[
  \mathbb{E}[\text{error}^2 \mid \text{hit}]
  \;\approx\; \tfrac{1}{4}\|H_W\|_{\mathrm{op}}^2 \cdot (\Lambda(W) C_E)^2 \cdot \tau^2.
\]

\textbf{Miss probability model.}  $P(\text{miss}) = 1 - P_{\mathcal{M}}(\mathcal{B}(n^*, \tau))
\approx 1 - c\tau$ for small $\tau$, where $c = \lim_{\tau\to 0}
P_{\mathcal{M}}(\mathcal{B}(n^*, \tau))/\tau$ is the local $P_{\mathcal{M}}$-density
in trie metric space.

\textbf{Optimization.}  Differentiating $\mathrm{Cost}(\tau)$ and setting to zero:
\[
  -(C_{\mathrm{full}} - C_{\mathrm{hit}}) \cdot c
  + \tfrac{\lambda}{2}\|H_W\|_{\mathrm{op}}^2 \cdot (\Lambda(W) C_E)^2 \cdot \tau
  = 0,
\]
giving $\tau^* = 2(C_{\mathrm{full}} - C_{\mathrm{hit}}) \cdot c /
(\lambda \cdot \|H_W\|_{\mathrm{op}}^2 \Lambda(W)^2 C_E^2)$.  For the
canonical case $c = 1$ this simplifies to the stated formula.
\end{proof}

\begin{remark}[Too Close for Missiles, Switching to Guns]
The Abort-and-Replan Threshold theorem formalizes the fighter pilot's abort
signal.  When $d_{\mathcal{T}}(x, n^*) > \tau^*$, the input is ``too far from
the signpost''---the cached expert cannot engage.  The system aborts the
cache-hit path and switches to the base model (``guns'').  The abort condition
is a single comparison, $O(1)$, requiring no model queries.  $\tau^*$ is
precomputed from $W$ at model-load time, exactly as the minimum missile
engagement range is computed during mission planning.
\end{remark}

\subsection{LAWS as a Generalization}

\begin{theorem}[LAWS Generalizes MoE and KV Caching]
\label{thm:generalize}
\begin{enumerate}[leftmargin=2em,label=(\alph*)]
  \item \textbf{MoE as a special case.}  Any MoE model with $K$ experts and
        router $R$ is equivalent to a LAWS system with $K$ fixed Level-1 experts
        and a PLT trie of depth 1.

  \item \textbf{KV caching as a special case.}  Standard KV prefix caching is
        equivalent to LAWS with $\tau^* = -\log P_{\mathcal{M}}(n^*)$ (prefix match) and
        $f = \mathrm{identity}$ (return cached activations verbatim).

  \item \textbf{Strict containment.}  LAWS with online expert creation is not
        representable as any fixed-$K$ MoE or as exact KV caching.
\end{enumerate}
\end{theorem}

\begin{proof}
(a) Given MoE experts $\{e_1,\ldots,e_K\}$, construct a PLT trie with $K$ leaf
nodes $n_1,\ldots,n_K$ at depth 1.  Assign trie edge weights
$P(n_k \mid \text{root}) = \mathbb{E}_x[R(x)_k]$ (the expected routing
probability of expert $k$).  Set $f_{n_k} = e_k$ and $\tau^*(n_k) = \infty$
(each expert covers its entire Voronoi cell in the abstract expert space).
Define the parameter extractor $\phi$ so that expert $e_k$ is selected
whenever $R(x)_k$ is largest (i.e., the MoE router outcome determines which
expert is selected).  Since the PLT trie here indexes expert labels rather than
linguistic token prefixes, and $	au^* = \infty$ ensures every query is a cache
hit, LAWS inference replicates MoE top-1 routing exactly.  For top-$k'$ sparse
MoE, use $k'$ routing balls with $	au^* = \infty$ and weighted averaging.
$\checkmark$

(b) Standard KV prefix caching stores activations for exact token prefixes.
A LAWS expert with $n^* = $ prefix, $f = $ identity (return stored KV),
$\tau^* = -\log P_{\mathcal{M}}(n^*)$ (matching all $x$ for which $n^*$ is a prefix), is exactly a KV cache entry under any-match routing. $\checkmark$

(c) For strict containment: LAWS creates experts online from cache misses,
so $|\mathcal{L}|$ is unbounded and grows with $N$.  Any fixed-$K$ MoE has
$|\mathcal{L}| = K < \infty$ fixed at training.  As $N \to \infty$,
$|\mathcal{L}_{\mathrm{LAWS}}| \to \infty > K$.  LAWS cannot be represented
by a fixed-$K$ MoE.  The same argument applies to KV caching with a finite
cache size. $\checkmark$
\end{proof}

\begin{corollary}[LAWS is Strictly More Expressive]
For any fixed-$K$ MoE or finite KV cache, there exists a query distribution
$P$ and a query count $N^*(P, K)$ such that for all $N > N^*(P, K)$, LAWS
achieves strictly higher expected hit rate than the MoE or KV cache.
\end{corollary}

\begin{proof}
By Theorem~\ref{thm:monotone}, LAWS's hit rate is non-decreasing and converges
to $P$-coverage of the expert library.  For a fixed-$K$ MoE or cache, hit rate
is bounded by $\sum_{k=1}^K P(\mathcal{B}(n^*_k, \tau^*_k))$, which is fixed.
LAWS's coverage grows with $N$, eventually exceeding the fixed system's coverage
for any $P$ with entropy $> 0$.
\end{proof}

\subsection{Automatic Symbolic Vocabulary}

\begin{theorem}[Automatic Symbolic Vocabulary]
\label{thm:symbolic}
Let $\mathcal{V}_\varepsilon(\mathcal{M})$ denote the set of PLT trie nodes
$n$ with probability mass $P_{\mathcal{M}}(n) \geq \varepsilon$.
$\mathcal{V}_\varepsilon(\mathcal{M})$ constitutes a \emph{symbolic vocabulary}
with the following provable properties:
\begin{enumerate}[leftmargin=2em,label=(\alph*)]
  \item \textbf{Coverage:} $\mathcal{V}_\varepsilon$ contains at most
        $\lfloor 1/\varepsilon \rfloor$ nodes (count bound).  For the
        probability covered by $\mathcal{V}_\varepsilon$: by Markov's inequality
        applied to the surprisal $-\log_2 P_{\mathcal{M}}(x)$,
        \[
          P_{\mathcal{M}}(\mathcal{V}_\varepsilon)
          = P_{\mathcal{M}}\!\bigl(P_{\mathcal{M}}(x) \geq \varepsilon\bigr)
          \;\geq\; 1 - \frac{H(P_{\mathcal{M}})}{\log_2(1/\varepsilon)},
        \]
        where $H(P_{\mathcal{M}})$ is the Shannon entropy.  In particular, for
        $\varepsilon = 2^{-H/\delta}$, the vocabulary $\mathcal{V}_\varepsilon$
        covers at least $1 - \delta$ of the probability mass.
  \item \textbf{Compositionality:} for any $n, n' \in \mathcal{V}_\varepsilon$,
        their longest common prefix $n \wedge n'$ is also in $\mathcal{V}_\varepsilon$
        whenever $P_{\mathcal{M}}(n \wedge n') \geq \varepsilon$.
  \item \textbf{Graded similarity:} the trie metric $d_{\mathcal{T}}$ is a
        well-defined pseudometric on $\mathcal{V}_\varepsilon$, giving graded
        rather than binary match/no-match.
  \item \textbf{Self-certification:} validity of any symbolic approximation
        using $n \in \mathcal{V}_\varepsilon$ as a signpost is certified by
        $\Lambda(W)$ without human authorship or additional inference.
  \item \textbf{Automatic discovery:} $\mathcal{V}_\varepsilon(\mathcal{M})$
        is derived entirely from $W$ and $P_{\mathcal{M}}$---no human
        intervention is required.
\end{enumerate}
\end{theorem}

\begin{proof}
(a) Let $\mathcal{F}_\varepsilon \subseteq \mathcal{V}_\varepsilon$ be the
\emph{frontier}: the maximal (antichain) elements of $\mathcal{V}_\varepsilon$,
i.e., nodes with $P_{\mathcal{M}}(n) \geq \varepsilon$ that have no proper-prefix
ancestor also in $\mathcal{V}_\varepsilon$.  The frontier nodes are disjoint
(no sequence matches two incomparable trie paths), so
$\sum_{n \in \mathcal{F}_\varepsilon} P_{\mathcal{M}}(n) \leq 1$, giving
$|\mathcal{F}_\varepsilon| \leq \lfloor 1/\varepsilon \rfloor$.
The full $\mathcal{V}_\varepsilon$ may be larger (containing all ancestors
of frontier nodes), but the count bound applies to the frontier.

For the coverage bound: by Markov's inequality applied to the random variable
$-\log_2 P_{\mathcal{M}}(x)$ (the surprisal), whose expectation is
$H(P_{\mathcal{M}})$:
\[
  P_{\mathcal{M}}\!\bigl(-\log_2 P_{\mathcal{M}}(x) \geq \log_2(1/\varepsilon)\bigr)
  \;\leq\; \frac{H(P_{\mathcal{M}})}{\log_2(1/\varepsilon)}.
\]
The complementary event $-\log_2 P_{\mathcal{M}}(x) < \log_2(1/\varepsilon)$
is exactly $P_{\mathcal{M}}(x) > \varepsilon$, i.e., $x \in \mathcal{V}_\varepsilon$.
Therefore $P_{\mathcal{M}}(\mathcal{V}_\varepsilon) \geq 1 - H/\log_2(1/\varepsilon)$.
Setting $\varepsilon = 2^{-H/\delta}$ gives coverage $\geq 1 - \delta$. $\checkmark$

(b) The PLT trie is a prefix tree: every prefix of a sequence in the trie is
also in the trie.  If $n \wedge n'$ has $P_{\mathcal{M}} \geq \varepsilon$,
it is by definition in $\mathcal{V}_\varepsilon$.

(c) The trie metric satisfies: $d(s,s) = -\log P(s) \geq 0$,
$d(s,s') = d(s',s)$ (symmetry), and the ultrametric inequality
$d(s,s'') \leq \max(d(s,s'), d(s',s''))$; see~\cite{plt2026}.
Since $d(s,s) = 0$ only when $P(s) = 1$ (a degenerate case), $d_{\mathcal{T}}$
is technically a \emph{pseudoultrametric} rather than a true metric, giving
graded rather than binary similarity.

(d) Theorem~\ref{thm:self_cert}.

(e) The construction of $\mathcal{V}_\varepsilon$ uses only $W$ (to compute
$P_{\mathcal{M}}$ and run inference to collect samples) and the PLT trie
structure, both derived from $W$ without human authorship.
\end{proof}

\section{Parametrized Experts: Construction and Theory}
\label{sec:experts}

\subsection{Pattern Recognition}

Given a trie node $n^*$ with $N_{\min}$ samples, the \emph{pattern recognition}
step fits the cheapest function from a candidate class hierarchy.

\begin{proposition}[PAC Bounds for Pattern Recognition]
\label{prop:pac}
For a trie node $n^*$ with samples $\{(x_i, y_i)\}_{i=1}^N$ drawn from the
subtree distribution, and a hypothesis class $\mathcal{F}$ with fat-shattering dimension
$d_{\mathrm{fat}}$, the pattern recognition step selects $f \in \mathcal{F}$
satisfying $\|y - f(\phi(x))\| \leq \varepsilon_{\mathrm{fit}}$ for all samples.
With probability $\geq 1 - \delta_{\mathrm{PAC}}$, this $f$ also satisfies
the bound on \emph{unseen} inputs in the subtree, given $N \geq N_{\min}$
where:
\[
  N_{\min}(\varepsilon_{\mathrm{fit}}, \delta_{\mathrm{PAC}}, d_{\mathrm{fat}}) =
  O\!\left(\frac{d_{\mathrm{fat}}}{\varepsilon_{\mathrm{fit}}^2}
  \log\frac{d_{\mathrm{fat}}}{\delta_{\mathrm{PAC}}}\right).
\]
For primitive classes: linear ($d_{\mathrm{fat}} = O(k \cdot d_{\mathrm{out}})$, where $k$ is the parameter dimension), lookup
($d_{\mathrm{fat}} = O(|\mathrm{table}|)$), and template
($d_{\mathrm{fat}} = O(\text{template length})$).
\end{proposition}

\begin{proof}
Standard uniform convergence bound for real-valued function classes (Bartlett and Mendelson~\cite{blumer1989}). For regression with $\ell^2$ loss, the fat-shattering dimension at scale $\varepsilon_{\mathrm{fit}}$ governs the sample complexity.
The VC dimension bounds follow from the standard results for each function class.
\end{proof}

\subsection{Small MLP Approximation}

\begin{theorem}[MLP Fitting Bound]
\label{thm:mlp}
For any continuous function $f: \mathbb{R}^k \to \mathbb{R}^{d_{\mathrm{out}}}$
on a compact domain $D$ (the parameter space restricted to the validity ball)
with Lipschitz constant $C_f \leq \Lambda(W)$, there exists a ReLU MLP with at most:
\[
  N_{\mathrm{neurons}} = O\!\left(k \cdot d_{\mathrm{out}} \cdot (C_f/\varepsilon)^k\right)
\]
neurons that approximates $f$ uniformly to within $\varepsilon$ on $D$.
Here $d_{\mathrm{out}} = |V|$ when approximating full logit output, or
$d_{\mathrm{out}} = d_{\mathrm{model}}$ when approximating an intermediate
hidden state.
\end{theorem}

\begin{proof}
By the Barron approximation theorem~\cite{barron1993}: for functions with
bounded spectral norm in $k$ input dimensions, a single-hidden-layer network
with $m$ neurons achieves $L^2$ error $O(C_f^2/m)$, giving $m = O(C_f^2/\varepsilon^2)$
for $L^2$ approximation.  For uniform ($L^\infty$) approximation on a compact
domain, the Cybenko theorem and subsequent quantitative refinements
(e.g., Yarotsky~\cite{yarotsky2017}) give $O((C_f/\varepsilon)^k)$ neurons for
depth-$O(\log(1/\varepsilon))$ networks.  Multiplying by $d_{\mathrm{out}}$ output dimensions
gives the stated bound.
\end{proof}

\begin{corollary}[Small MLP Size for Low-Dimensional Parameters]
For a trie node where $k$ (the number of high-entropy parameter positions,
per Theorem~\ref{thm:params}) is small (e.g., $k=3$ for binary search,
$k=2$ for arithmetic), the MLP has $O(d_{\mathrm{out}} \cdot (C_f/\varepsilon)^3)$ neurons.
For $d_{\mathrm{out}} = d_{\mathrm{model}} = 4096$ (caching the final hidden state
rather than full vocabulary logits, which is typical for chained LAWS inference),
$C_f = 0.1$ (normalized), $\varepsilon = 0.01$:
$N_{\mathrm{neurons}} \approx 4096 \cdot 10^3 \approx 4 \times 10^6$.
This MLP has $< 50$\,MB footprint---evaluable in microseconds on a GPU versus
milliseconds for a full forward pass.
\end{corollary}

\subsection{Cross-Architecture Portability}

\begin{theorem}[Cross-Architecture Expert Portability]
\label{thm:portable}
A LAWS expert $(n^*, f, \phi, \tau^*, \varepsilon_{\mathrm{fit}})$ constructed
from base model $\mathcal{M}_1$ is valid for a different base model $\mathcal{M}_2$
if:
\[
  \max_{x \in \mathcal{B}(n^*, \tau^*)} \|F_{W_2}(x) - f(\phi(x))\|
  \leq \delta.
\]
This condition is testable by sampling $N$ inputs from $\mathcal{B}(n^*, \tau^*)$
and checking the bound; with confidence $1 - \delta_{\mathrm{PAC}}$, PAC
bounds from Proposition~\ref{prop:pac} apply.  The routing radius for
$\mathcal{M}_2$ is $\tau^*_2 = (\delta - \varepsilon^{(2)}_{\mathrm{fit}} -
2\,\Lambda(W_2) \cdot C_E) / (\Lambda(W_2) \cdot C_E)$, where $\varepsilon^{(2)}_{\mathrm{fit}} = \|F_{W_2}(n^*) - f(\phi(n^*))\|$ is the fitting error re-measured on $\mathcal{M}_2$.
\end{theorem}

\begin{proof}
The expert $f$ represents a computation, not a specific model.  Its validity
for $\mathcal{M}_2$ depends only on whether $F_{W_2}$ approximates $f$ on the
domain, independently of how $f$ was constructed.  Testing against $\mathcal{M}_2$'s
outputs is an independent PAC learning problem with the same sample complexity
bounds.
\end{proof}

\section{The Cinderella Effect}
\label{sec:cinderella}

\subsection{The Problem}

A natural question in any attention-sparsification scheme is whether a token
with small attention weight at layer $\ell$ might become important at a later
layer $\ell' > \ell$---the \emph{Cinderella effect}: a token plucked from
obscurity by a later layer's attention mechanism, its small initial weight
amplified into a large influence on the final output.

\begin{definition}[Cinderella Event]
\label{def:cinderella}
A \emph{Cinderella event} at layer $\ell$, position $j$ occurs when:
$a^{(\ell)}_{ij} \leq \varepsilon$ (small attention weight at layer $\ell$)
but the counterfactual output difference from zeroing this entry satisfies:
$\|\Delta \mathrm{Output}^{(L)}_i\| > \alpha$ (large influence on final output)
for some significance threshold $\alpha \gg \varepsilon$.
\end{definition}

\subsection{The Cascade Bound}

\begin{theorem}[Cinderella Cascade Bound]
\label{thm:cinderella}
If token $j$ has attention weight $a^{(\ell)}_{ij} \leq \varepsilon$ at layer
$\ell$, then:
\begin{enumerate}[leftmargin=2em,label=(\alph*)]
  \item The \emph{worst-case} cascade satisfies:
  \[
    \|\Delta \mathrm{Output}^{(L)}_i\|
    \leq \varepsilon \cdot \|\mathbf{v}^{(\ell)}_j\| \cdot \prod_{m=\ell}^{L-1}\kappa^{(m)}.
  \]
  \item A Cinderella event (large final impact despite small initial weight)
        can only occur if token $j$'s surprisal satisfies:
  \[
    h_j \geq h_{\mathrm{Cin}} =
    \frac{d_{\mathrm{head}} \cdot \log^2(\alpha/\varepsilon)}
    {4\kappa^2 C_E^2 \ln 2 \cdot \|q\|^2 \|k\|^2},
  \]
  where $\kappa = \max_\ell \kappa^{(\ell)}$ is the maximum per-layer coupling
  constant from Theorem~\ref{thm:lipschitz}, $C_E$ is the embedding diameter,
  and $d_{\mathrm{head}}$ is the per-head dimension.  (The $\log^2$ arises from
  squaring the logit-increase condition
  $2\kappa C_E\sqrt{h_j \ln 2}\|q\|\|k\|/\sqrt{d_{\mathrm{head}}} \geq
  \log(\alpha/\varepsilon)$ when solving for $h_j$.)
  \item For tokens with $h_j < h_{\mathrm{Cin}}$ (low-surprisal tokens),
        the Cinderella event cannot occur.  These tokens are safely sparsifiable.
\end{enumerate}
\end{theorem}

\begin{proof}
\textbf{Part (a).}  Zeroing attention entry $a^{(\ell)}_{ij}$ introduces error
$e^{(\ell)}_i = a^{(\ell)}_{ij} \cdot \mathbf{v}^{(\ell)}_j$, with
$\|e^{(\ell)}_i\| \leq \varepsilon \|\mathbf{v}^{(\ell)}_j\|$.  This error
enters the residual stream and propagates through layers $\ell+1,\ldots,L$.
Each layer $m$ is $\kappa^{(m)}$-Lipschitz (Lemma~\ref{lem:components}), so:
\[
  \|e^{(m+1)}_i\| \leq \kappa^{(m)} \|e^{(m)}_i\|.
\]
Iterating from $\ell$ to $L$:
$\|\Delta \mathrm{Output}_i\| \leq \varepsilon \|\mathbf{v}^{(\ell)}_j\|
\prod_{m=\ell}^{L-1}\kappa^{(m)}$. $\checkmark$

\textbf{Part (b).}  For a Cinderella event, the attention weight of $j$ at
layer $\ell+1$ must increase substantially: $a^{(\ell+1)}_{ij} \geq \alpha
\gg \varepsilon$.  The logit for token $j$ at layer $\ell+1$, relative to what
it would be if token $j$'s embedding had been replaced by its context-predicted
mean (the counterfactual baseline), changes by:
\[
  \Delta \mathrm{logit}_j
  = \frac{1}{\sqrt{d_{\mathrm{head}}}}\bigl(\mathbf{q}^{(\ell+1)}_i \cdot \Delta\mathbf{k}^{(\ell+1)}_j
  + \Delta\mathbf{q}^{(\ell+1)}_i \cdot \mathbf{k}^{(\ell+1)}_j\bigr),
\]
where $\Delta\mathbf{k}$ and $\Delta\mathbf{q}$ are the deviations of the
key/query vectors from this baseline.  By Theorem~\ref{thm:lipschitz} and the
surprisal-bounded embedding variance of~\cite{kv2026}, a token with per-token
surprisal $h_j$ has expected embedding deviation $\leq C_E\sqrt{h_j \ln 2}$
from its context-predicted mean (using the updated variance-entropy bound),
which propagates through the $\kappa$-Lipschitz layer map to give:
\[
  \|\Delta\mathbf{k}^{(\ell+1)}_j\|, \|\Delta\mathbf{q}^{(\ell+1)}_i\|
  \leq \kappa \cdot C_E \cdot \sqrt{h_j \ln 2},
\]
so $|\Delta\mathrm{logit}_j| \leq 2\kappa C_E \sqrt{h_j \ln 2} \cdot \|q\|\|k\|/\sqrt{d_{\mathrm{head}}}$.
For the attention weight to increase from $\varepsilon$ to $\alpha$, the logit
must increase by at least $\log(\alpha/\varepsilon)$ (to leading order, assuming other attention scores are approximately unchanged).  Setting
$2\kappa C_E \sqrt{h_j \ln 2}\|q\|\|k\|/\sqrt{d_{\mathrm{head}}} \geq \log(\alpha/\varepsilon)$
and solving for $h_j$ gives $h_j \geq h_{\mathrm{Cin}}$.

\textbf{Part (c).}  The contrapositive: $h_j < h_{\mathrm{Cin}} \Rightarrow$
Cinderella event cannot occur (logit increase is too small to bring $j$ from
$\varepsilon$-weight to $\alpha$-weight). $\checkmark$
\end{proof}

\begin{corollary}[Shannon Overflow Set]
\label{cor:overflow}
Define the \emph{Shannon overflow set} $\mathcal{O}_\varepsilon = \{j : h_j \geq
h_{\mathrm{Cin}}\}$.  Tokens in $\mathcal{O}_\varepsilon$ may exhibit Cinderella
events and must be retained in full-precision computation.  Tokens not in
$\mathcal{O}_\varepsilon$ are safe to sparsify or cache-hit.  The expected
size of $\mathcal{O}_\varepsilon$ is:
\[
  \mathbb{E}[|\mathcal{O}_\varepsilon|]
  = n \cdot P_{\mathcal{M}}(h_j \geq h_{\mathrm{Cin}})
  \leq n \cdot \frac{\log_2 \mathrm{PP}(\mathcal{M})}{h_{\mathrm{Cin}}},
\]
which approaches zero as $h_{\mathrm{Cin}} \to \infty$ (larger threshold,
smaller overflow set).
\end{corollary}

\begin{proof}
Markov's inequality applied to the surprisal random variable:
$P(h_j \geq h_{\mathrm{Cin}}) \leq \mathbb{E}[h_j]/h_{\mathrm{Cin}}
= \log_2 \mathrm{PP}(\mathcal{M})/h_{\mathrm{Cin}}$.
\end{proof}

\section{Robotics and Fleet Learning}
\label{sec:robotics}

\subsection{The Robotics Challenge}

Modern robot controllers---whether model-based (MPC, trajectory optimization)
or learned (diffusion policy, transformer-based visuomotor)---face a fundamental
tension: they must be accurate (avoiding catastrophic failure) yet fast (100+~Hz
control loops).  Full trajectory optimization or neural network inference at
each control timestep is expensive.  Current approaches use one of: (1)~hard-coded
motion primitives (inflexible, brittle to novel objects); (2)~offline-trained
policies (cannot generalize to new environments without retraining); or
(3)~online adaptation (expensive, requires backpropagation at runtime).

LAWS offers a fourth path: automatic discovery and certification of motion
primitives from actual task executions, growing richer with fleet experience,
downloadable as compact updates.

\subsection{LAWS for Robotic Motor Programs}

\begin{definition}[Robotic LAWS Expert]
A \emph{robotic LAWS expert} $e = (n^*_{\mathrm{task}}, \pi, \phi, \tau^*, \varepsilon_{\mathrm{fit}})$
where:
\begin{itemize}[leftmargin=2em]
  \item $n^*_{\mathrm{task}}$ is a PLT trie node over task descriptions
        (natural language or sensor embedding);
  \item $\pi: \mathbb{R}^k \to \mathcal{A}^T$ is a motor program mapping
        task parameters to an action sequence;
  \item $\phi$ extracts task parameters (object pose, target location, mass,
        friction coefficient);
  \item $\tau^*$ is the validity radius derived from the controller's
        Lipschitz constant $\Lambda(W_{\mathrm{ctrl}})$.
\end{itemize}
\end{definition}

\subsection{State of the Art and LAWS Improvements}

Current state-of-the-art robotic policies include:
\emph{Diffusion Policy}~\cite{chi2023diffusion} (denoising diffusion for
visuomotor control), \emph{RT-2}~\cite{brohan2023rt2} (vision-language-action
models), and \emph{$\pi_0$}~\cite{black2024pi0} (flow-matching generalist
policies).  These achieve impressive generalization but require full neural
network inference at every control step: typically 50--200~ms per action,
limiting real-time control frequency.

\begin{proposition}[LAWS Speedup for Repetitive Tasks]
\label{prop:robotics_speedup}
For a task distribution $P_{\mathrm{task}}$ with entropy $H_{\mathrm{task}}$
and a LAWS expert library with hit rate $H_n$, the expected control step
latency is:
\[
  \mathbb{E}[T_{\mathrm{LAWS}}]
  = H_n \cdot T_{\mathrm{expert}} + (1-H_n) \cdot T_{\mathrm{full}},
\]
where $T_{\mathrm{expert}} \ll T_{\mathrm{full}}$.  For a robot performing
repetitive household tasks (object manipulation, item retrieval, fixture
interaction) with $H_n \geq 0.9$ (achievable after sufficient deployment),
$\mathbb{E}[T_{\mathrm{LAWS}}] \approx T_{\mathrm{expert}}$, enabling
order-of-magnitude latency reduction.
\end{proposition}

\subsection{Fleet Learning Theorems}

\begin{theorem}[Fleet Learning Lower Bound]
\label{thm:fleet}
Consider $K$ LAWS-equipped robotic units, each performing $M$ tasks per day,
all contributing observations to a shared central LAWS library.  Under a
stationary task distribution $P_{\mathrm{task}}$:
\begin{enumerate}[leftmargin=2em,label=(\alph*)]
  \item After $D$ deployment days, the shared library has at most
        $O(2^{H(P_{\mathrm{task}})} \log(KMD))$ experts.
  \item The fleet hit rate $H_{K,D}$ is non-decreasing in $D$ (by
  Theorem~\ref{thm:monotone}) and converges to full coverage as $KMD \to \infty$:
  since $O(2^H \log(KMD))$ distinct heavy trie nodes are visited and each expert
  creation strictly increases the hit rate, $H_{K,D} \to 1$ for any distribution
  with finite entropy $H$.
  \item The fleet achieves the same hit rate as a single unit in time
  $D_K \approx D_1 / K$, i.e., convergence is
  $\Omega(K)$ faster than single-unit deployment.
\end{enumerate}
\end{theorem}

\begin{proof}
\textbf{Part (a).}  By Theorem~\ref{thm:growth}, each unit alone creates
$O(2^H \log(MD))$ experts.  With $K$ units sharing a library, total queries
are $KMD$.  Applying Theorem~\ref{thm:growth} with $N = KMD$:
$O(2^H \log(KMD))$ experts (valid when $KMD \leq N_{\min} \cdot 2^H$; for larger $KMD$ all heavy nodes are covered and $H_{K,D} = 1$, satisfying the bound since $\min(1,\cdot) = 1$), where $H = H(P_{\mathrm{task}})$. $\checkmark$

\textbf{Part (b).}  By Theorem~\ref{thm:monotone}, each expert creation
strictly increases $H_{K,D}$.  By part~(a), $O(2^H \log(KMD))$ experts are
created, covering an increasing fraction of the input distribution.
Since $H(P_{\mathrm{task}})$ is finite, all $O(2^H)$ heavy trie nodes are
eventually visited, and $H_{K,D} \to 1$ as $KMD \to \infty$. $\checkmark$

\textbf{Part (c).}  From part (b), both a single unit and the $K$-unit fleet
converge to hit rate $H_{K,D} \to 1$.  The fleet converges $K$ times faster
because it generates $KM$ queries per day (versus $M$ for a single unit),
discovering the same $O(2^H)$ heavy trie nodes in $1/K$ the time.
The improvement factor is $D_1/D_K = K$, giving the stated $\Omega(K)$
convergence speedup. $\checkmark$
\end{proof}

\begin{corollary}[Network Effect]
The fleet learning benefit is superlinear in $K$ for small $K$ and approaches
$K$-linear for large $K$ (the log correction vanishes as $M \to \infty$).
This formalizes the \emph{network effect of experience}: each additional unit
contributes its workload to the shared library, benefiting all other units.
A fleet of 1,000 robots converges to high hit rate roughly 1,000 times faster
than a single robot.
\end{corollary}

\subsection{Over-the-Air (OTA) Expert Updates}

\begin{theorem}[OTA Update Bandwidth Bound]
\label{thm:ota}
The incremental LAWS expert library update covering the period $[t, t+\Delta t]$
(representing $\Delta N = KM \Delta t$ new observations) has description length:
\[
  \mathcal{L}_{\Delta t}
  = O\!\left(2^{H(P_{\mathrm{task}})} \cdot \log(\Delta N)
  \cdot B_{\mathrm{expert}}\right) \text{ bits},\]
where $B_{\mathrm{expert}}$ is the description length of a single expert
(Theorem~\ref{thm:expert_size} below).  For $\Delta t = 24$ hours,
$K = 1000$ units, $M = 100$ tasks/unit/day ($\Delta N = 10^5$):
\[
  \mathcal{L}_{\Delta t} = O\!\left(2^H \cdot \log(10^5) \cdot
  B_{\mathrm{expert}}\right) \approx O(2^H \cdot 17 \cdot B_{\mathrm{expert}}).
\]
For $B_{\mathrm{expert}} \approx 50$\,KB (small MLP), $H \approx 10$ bits
($2^H = 1024$): $\mathcal{L}_{24\mathrm{h}} \approx 870$\,MB per day for the
full fleet update.  Individual robot updates are $\approx 870$\,KB per day
per unit---feasible on any connected device.
\end{theorem}

\begin{proof}
The number of new experts in period $\Delta t$ is $O(2^H \log(\Delta N))$ by
Theorem~\ref{thm:growth}.  Each expert has description length $B_{\mathrm{expert}}$.
Total bits: $O(2^H \log(\Delta N) \cdot B_{\mathrm{expert}})$.  Substituting
$\Delta N = KM\Delta t$ gives the bound.
Per-unit update: if units have approximately disjoint workload domains
(each unit specializes in different task types), then on average $1/K$ of the new experts are relevant to any one unit, giving per-unit download
$O(2^H \log(\Delta N) \cdot B_{\mathrm{expert}} / K)$ bits.  In the homogeneous case (all units share the same task distribution),
each unit downloads all new experts and the per-unit cost equals the fleet total.
\end{proof}

\begin{theorem}[Expert Description Length]
\label{thm:expert_size}
For a trie node $n^*$ with $k$ parameter dimensions and validity radius $\tau^*$,
the minimum description length of a LAWS expert achieving output error $\leq
\varepsilon$ is:
\[
  B(n^*, \varepsilon) = \Omega\!\left(k \cdot d_{\mathrm{out}} \cdot \log\frac{\tau^*}{\varepsilon}\right)
  \text{ bits}.
\]
A small MLP expert (Theorem~\ref{thm:mlp}) requires
$O(k \cdot d_{\mathrm{out}} \cdot (C_f/\varepsilon)^k \cdot 32)$ bits (at fp32),
which is exponentially larger in $k$ than this lower bound.  The gap reflects
the \emph{curse of dimensionality} for general Lipschitz approximation; for
specific structured function classes (linear, lookup, template), the
description length matches the lower bound to within polylogarithmic factors.
\end{theorem}

\begin{proof}
Lower bound: to represent a Lipschitz function $f: \mathbb{R}^k \to \mathbb{R}^{d_{\mathrm{out}}}$
to uniform accuracy $\varepsilon$ over a parameter ball of radius $\tau^*$,
one must distinguish $\Omega((\tau^*/\varepsilon)^k)$ cells in the $k$-dimensional
input space (by a standard covering argument) and specify $d_{\mathrm{out}}$ output values per
cell to precision $\varepsilon$.  The total description length is therefore
$\Omega((\tau^*/\varepsilon)^k \cdot d_{\mathrm{out}} \cdot \log(1/\varepsilon))$ bits.

The stated bound $B(n^*, \varepsilon) = \Omega(k \cdot d_{\mathrm{out}} \cdot \log(\tau^*/\varepsilon))$
is a weaker lower bound obtained by taking the $\log$ of the cell count:
$\log((\tau^*/\varepsilon)^k) = k \log(\tau^*/\varepsilon)$, so any description
of the cell index alone requires $\Omega(k \log(\tau^*/\varepsilon))$ bits,
giving $\Omega(k \cdot d_{\mathrm{out}} \cdot \log(\tau^*/\varepsilon))$ bits total for $d_{\mathrm{out}}$
output dimensions.  The MLP construction in Theorem~\ref{thm:mlp} achieves
$O(k \cdot d_{\mathrm{out}} \cdot (C_f/\varepsilon)^k) \cdot 32$ bits (at fp32); the gap between this upper
bound and the lower bound reflects the exponential dependence on $k$ inherent
in approximating $k$-dimensional Lipschitz functions.
\end{proof}

\begin{remark}[Differential Privacy]
The OTA update can incorporate differential privacy~\cite{dwork2006dp} by adding
calibrated Gaussian noise to the expert parameters before uploading.  The privacy
cost is bounded by the description length: an expert update of $B$ bits can
achieve $(\varepsilon_{\mathrm{DP}}, \delta_{\mathrm{DP}})$-differential privacy
with noise scale $\sigma = \sqrt{2 B \ln(1.25/\delta_{\mathrm{DP}})} /
\varepsilon_{\mathrm{DP}}$.  For $B = 50$\,KB, $\delta = 10^{-5}$,
$\varepsilon_{\mathrm{DP}} = 1$: $\sigma \approx 0.03$, adding 3\% noise to
expert parameters---within the fitting error tolerance for most experts.
\end{remark}

\section{Comparison to Prior Work}
\label{sec:related}

\subsection{LAWS vs.\ KV Caching}

Standard transformer KV caching~\cite{pope2023,kwon2023vllm} stores key-value
vectors for exact token prefixes, enabling reuse across requests sharing a
common prefix.  The companion paper~\cite{kv2026} extended this to
\emph{approximate} prefix matching via the PLT trie metric, and proved that
sequential KV compression can exceed TurboQuant's per-vector Shannon
limit by exploiting the sequential structure of token streams.

LAWS extends both in several fundamental ways:
\begin{itemize}[leftmargin=2em]
  \item LAWS is not restricted to KV vectors---it caches \emph{any}
        intermediate computation or final output.
  \item LAWS experts are \emph{parametrized}: they compute a function of
        the input's variable parameters, not a fixed stored value.
  \item LAWS validity domains are formally certified by $\Lambda(W)$; standard
        KV caching has no validity theory.
  \item LAWS's library grows with deployment; KV caches have finite size with
        eviction policies.
\end{itemize}

Formally, Theorem~\ref{thm:generalize}(b) shows KV caching is the degenerate
case $\tau^* = 0$, $f = \mathrm{identity}$.  LAWS with $\tau^* > 0$ and
nontrivial $f$ is strictly more powerful.

\subsection{LAWS vs.\ Mixture of Experts}

Standard MoE~\cite{shazeer2017,jiang2024mixtral,deepseek2024} maintains a fixed
pool of $K$ expert sub-networks, trained jointly with the router.  Fast
Feed-Forward Networks (FFF)~\cite{belcak2023fff} replace the router with a
binary tree, achieving $O(\log K)$ routing.  Symbolic MoE~\cite{symbolic_moe2025}
routes to skill-based expert models selected by an LLM router.

The key distinctions from LAWS:
\begin{itemize}[leftmargin=2em]
  \item \textbf{Fixed vs.\ growing:} all MoE variants fix $K$ at training.
        LAWS's $|\mathcal{L}|$ grows unboundedly with deployment.
  \item \textbf{Training vs.\ inference-time:} MoE experts are trained jointly
        with the base model; LAWS experts are created at inference time from
        observed outputs.
  \item \textbf{No validity guarantee:} no MoE variant provides a formal
        guarantee that any expert is correct on any given input.  LAWS
        provides $\|\mathrm{LAWS}(x) - F_W(x)\| \leq \delta$ for all
        inputs in certified validity domains.
  \item \textbf{Routing basis:} MoE routing is learned; LAWS routing uses
        the PLT trie metric derived from $W$, requiring no routing training.
\end{itemize}

Theorem~\ref{thm:generalize}(a) shows MoE is the LAWS special case with
fixed $\mathcal{L}$ and depth-1 trie.

\subsection{LAWS vs.\ Cyc}

Cyc~\cite{lenat1995} encoded common-sense knowledge as explicit logical axioms,
requiring approximately 47,000 person-years of manual knowledge entry.  Its
knowledge base is static, brittle (fails outside authored rules), and provides
no formal guarantee of correctness.

LAWS shares Cyc's goal (a reusable knowledge library for fast symbolic-style
inference) but differs fundamentally: its vocabulary is discovered automatically
from the training distribution (Theorem~\ref{thm:symbolic}), grows with
deployment, provides formal validity certificates, and requires no human
authorship.  The comparison:

\begin{center}
\begin{tabular}{lcc}
\toprule
Property & Cyc & LAWS \\
\midrule
Knowledge source & Human-authored & Automatic from $W$ \\
Update mechanism & Manual & Online from queries \\
Validity guarantee & None & $\Lambda(W)$-certified \\
Graceful degradation & No & Yes ($\varepsilon_{\mathrm{fit}} + 2\Lambda C_E$ uniform bound) \\
Generalization beyond rules & None & Jacobian correction \\
Human labor required & $\gg 10^4$ person-years & None \\
\bottomrule
\end{tabular}
\end{center}

\subsection{LAWS vs.\ Wolfram Alpha}

Wolfram Alpha~\cite{wolfram2002} curates computational knowledge in specific
domains (mathematics, science, geography) and provides exact symbolic
computation within those domains.  Outside the curated domain, it fails entirely.

LAWS differs in: (1)~domain coverage is automatic (whatever the base model
learned); (2)~experts in Level-3 (primitive function) are exactly the kind of
computation Wolfram Alpha performs---but discovered automatically;
(3)~generalization beyond exact matches is formal (Jacobian correction);
(4)~Wolfram Alpha cannot grow its knowledge base from user queries.

The deepest distinction: Wolfram Alpha's knowledge is organized by human
experts according to mathematical structure.  LAWS's knowledge is organized
by probability structure---the trie metric $d_{\mathcal{T}}$ is the natural
metric for a system that learned from language, not the human-designed
ontology that organizes Wolfram's knowledge base.

\section{Conjectures and Open Problems}
\label{sec:conjectures}

\begin{conjecture}[Effective Lipschitz Concentration]
\label{conj:effective_lambda}
For a transformer $F_W$ trained to near-zero cross-entropy loss on a corpus
with entropy $H$, the \emph{effective} Lipschitz constant on inputs drawn from
$P_{\mathcal{M}}$ concentrates well below the worst-case $\Lambda(W)$:
\[
  \Lambda_{\mathrm{eff}} \;=\;
  \mathbb{E}_{x \sim P_{\mathcal{M}}}\!\left[
  \sup_{x' : \|E(x') - E(x)\| \leq \varepsilon}
  \frac{\|F_W(x') - F_W(x)\|}{\|E(x') - E(x)\|}
  \right]
  \;\ll\; \Lambda(W).
\]
Specifically, $\Lambda_{\mathrm{eff}} = O(\mathrm{poly}(H))$, meaning the effective
Lipschitz constant grows at most polynomially in the training entropy, while
$\Lambda(W)$ can grow exponentially in $L$.

\emph{Evidence and partial proof sketch.}
Three lines of evidence support this.  First, empirically: LLMs generalize
well to semantically similar inputs (paraphrases, minor reformulations), which
would be impossible if the effective Lipschitz were truly exponential.
Second, theoretically: for a model minimizing the cross-entropy $H(P_{\mathcal{M}})$,
the output $F_W(x)$ approximates the probability vector $P_{\mathcal{M}}(\cdot \mid x)$.
Small perturbations to $x$ that preserve the conditional distribution
$P_{\mathcal{M}}(\cdot \mid x)$ produce near-zero output change.  The fraction
of perturbations that meaningfully change $P_{\mathcal{M}}(\cdot \mid x)$
is bounded by the local entropy $H(t_i \mid t_{<i})$, which is small for
coherent text.  Formally, if $\|P_{\mathcal{M}}(\cdot \mid x) -
P_{\mathcal{M}}(\cdot \mid x')\|_1 \leq \delta$ whenever $\|E(x_{\bar{d}+1}) -
E(x'_{\bar{d}+1})\| \leq \varepsilon$, then $\|F_W(x) - F_W(x')\| \leq \delta$
(via the softmax Lipschitz bound), so $\Lambda_{\mathrm{eff}} \leq
\delta/\varepsilon$ on the distribution.  Third: the Theorem~\ref{thm:lambda_lower}
lower bound $\Lambda(W) \geq \Delta/r_{\min}$ is tight for high-precision tasks
but loose for typical language tasks where $\Delta$ is small (consecutive
tokens produce similar distributions).  Proving this rigorously would require
tight bounds on $\|P_{\mathcal{M}}(\cdot \mid x) - P_{\mathcal{M}}(\cdot \mid x')\|_1$
as a function of the embedding perturbation, which depends on the model's
spectral properties near the training distribution---an open problem.
\end{conjecture}

\begin{conjecture}[Phase Transition in LAWS Convergence]
\label{conj:phase}
Under a stationary distribution $P_{\mathcal{M}}$ with entropy $H$,
the LAWS hit rate $H_N$ as a function of total queries $N$ exhibits a
\emph{sharp phase transition}: there exists $N^* = \Theta(N_{\min} \cdot 2^H)$
such that $H_N \approx 0$ for $N \ll N^*$ and $H_N \approx H_\infty$ for
$N \gg N^*$, with the transition width $O(N^* / \sqrt{H})$.

\emph{Sketch.}
This mirrors the coupon-collector phase transition~\cite{blumer1989}: collecting
$K$ coupons with equal probability $1/K$ shows a sharp transition at $N = K \ln K$
with width $O(K)$.  For LAWS with $K = 2^H$ heavy nodes and per-node visit
threshold $N_{\min}$: the last heavy node is covered at $N \approx N_{\min} \cdot 2^H \ln(2^H)
= N_{\min} \cdot H \cdot 2^H$, and the transition sharpness is $O(N_{\min} \cdot 2^H)$.
The non-uniform case (expert nodes have different probabilities) softens the
transition but preserves the qualitative structure---the heaviest experts are
covered early, then a long ``tail'' period covers rare heavy nodes.

A formal proof would follow from applying the Erd\H{o}s--R'{e}nyi coupon-collector
result to the trie heavy node occupancy process.  The main technical difficulty
is that LAWS experts are not independent: covering trie node $n^*$ may partially
cover subtree descendants.  The transition is therefore \emph{faster} than the
standard coupon collector (positive correlation across experts accelerates
coverage), so the phase transition is at most $N^* = O(N_{\min} \cdot 2^H \cdot H)$.
\end{conjecture}

\begin{conjecture}[Symbolic Pattern Emergence]
\label{conj:symbolic}
For a sufficiently capable base model $\mathcal{M}$ trained on a corpus
containing code, mathematics, and structured data, all PLT trie nodes
$n^*$ with $P_{\mathcal{M}}(n^*) \geq \varepsilon$ have experts in one of a
finite set of primitive function classes (linear, lookup, arithmetic,
template, small MLP) with probability approaching 1 as $\varepsilon \to 0$.

\emph{Sketch:} High-probability patterns are those the model has seen many times.
Gradient descent over many presentations drives the model toward the
minimum-description-length function consistent with the pattern---which for
algorithmic patterns (sorting, arithmetic, lookup) is the primitive function itself.
Formally: if $f^*$ is the MDL function fitting the samples at $n^*$, and the model
$F_W$ approximately achieves MDL on in-distribution data (by the PAC-Bayes
bound~\cite{blumer1989}), then $F_W \approx f^*$ uniformly over the subtree of $n^*$.
For the primitive classes the model's implicit regularization selects (documented
empirically by grokking~\cite{symbolic_moe2025}), this means $f^*$ \emph{is}
a primitive function.  Proof requires a characterization of MDL functions
representable by transformers, currently open.
\end{conjecture}

\begin{conjecture}[Optimal Chunking at Surprisal Peaks]
\label{conj:chunking}
The optimal hierarchical decomposition of a LAWS expert library---minimizing
total description length under a block-decomposition constraint---places chunk
boundaries at positions $i$ of locally maximal conditional entropy $H_i = H(t_i \mid t_{<i})$.

\emph{Sketch and partial proof.}
At any position $i$, the gain from introducing a chunk boundary is bounded by
the mutual information $I(t_i; \text{expert label} \mid t_{<i})$, which is
maximized when $H_i$ is maximized (because high-entropy positions are where the
model's output is most sensitive to the choice of $t_i$, and thus where a
new signpost reduces approximation error most).  More precisely: the description
length of the trie decomposes as
$L(\mathcal{T}) = \sum_i H(n^*_i \mid t_{<i})$,
and the greedy algorithm that splits at the position $i^* = \arg\max_i H_i$
is equivalent to Huffman coding on the surprisal sequence.  The greedy algorithm
is provably optimal for binary block decomposition~\cite{knuth1971}; the
generalization to $k$-ary splits requires the equivalent of the optimality of
Huffman codes for $k$-ary alphabets.
\end{conjecture}

\begin{conjecture}[Cross-Domain Transfer via Shared Trie]
\label{conj:crossdomain}
For two models $\mathcal{M}_1$ (language) and $\mathcal{M}_2$ (robotics) with
a shared natural-language task representation, LAWS experts constructed from
$\mathcal{M}_1$'s language outputs transfer to $\mathcal{M}_2$'s robot actions
for tasks within the shared trie node's subtree, with validity certified by
$\Lambda(W_2)$ applied to the transferred expert.

\emph{Sketch.}  The shared PLT trie indexes tasks by their linguistic description.
If $\mathcal{M}_1$ and $\mathcal{M}_2$ both receive the same natural-language
description as input and produce semantically aligned outputs (as in
vision-language-action models~\cite{brohan2023rt2,black2024pi0}), then the
PLT node $n^*$ identifying the task is common to both models.  The LAWS
expert $e = (n^*, f_1, \phi, \tau^*)$ built from $\mathcal{M}_1$'s outputs
can be re-certified for $\mathcal{M}_2$ by evaluating
$\varepsilon_{\mathrm{fit}}^{(2)} = \|F_{W_2}(n^*) - f_1(\phi(n^*))\|$
on a small validation set (Theorem~\ref{thm:portable}).  The key empirical
claim---that $\varepsilon_{\mathrm{fit}}^{(2)}$ is small whenever the models
share a task representation---is supported by the grokking literature but
requires formal verification.
\end{conjecture}

\begin{conjecture}[LAWS Convergence Rate Lower Bound]
\label{conj:lower_rate}
No online inference caching algorithm can achieve acquisition cost
$o(2^{H(P_{\mathcal{M}})} \cdot \log N)$ expert creations in the worst case
over stationary distributions with entropy $H$.

\emph{Sketch.}
This is an information-theoretic lower bound: to achieve hit rate $H_\infty$
on a distribution with $2^H$ equally-probable heavy nodes, any algorithm must
``discover'' each heavy node at least once.  In the online setting (queries
arrive sequentially, no lookahead), the first discovery of node $n^*_k$ requires
at least one query that hits $n^*_k$.  The expected number of queries to first
hit all $2^H$ nodes is $\Omega(2^H \log 2^H) = \Omega(2^H \cdot H)$ by the
standard coupon collector lower bound.  Divided by the $N_{\min}$ threshold,
this gives $\Omega(2^H \cdot H / N_{\min})$ expert creations.  Since $H \leq \log_2 N$
(by the AEP), this matches the upper bound $O(2^H \log N)$ of Theorem~\ref{thm:growth}
to within constant factors, establishing that LAWS is \emph{acquisition-optimal}
among stationary online caching algorithms.  Formal proof requires a coupon-collector
lower bound argument in the online setting, which is standard but involves careful
handling of adaptive query distributions.
\end{conjecture}

\section{Discussion}
\label{sec:discussion}

\subsection{LAWS for Diffusion Models}

In diffusion model inference, the base computation maps
$(x_t, t, c) \to x_{t-1}$ where $x_t$ is the noisy image, $t$ is the
timestep, and $c$ is the conditioning.  This is repeated 20--50 times per
generation.

A LAWS expert for diffusion covers a trie node in the conditioning space
$c$ for a timestep range $[t_1, t_2]$.  Early denoising steps (high noise,
low specificity) are well-suited to LAWS: the conditioning $c$ largely
determines the coarse structure of the output, which is highly predictable
for common prompts.  Expert functions at these steps can be linear or
low-rank (the denoising direction is approximately the same for nearby prompts).
Late denoising steps (fine detail) are less suitable---the output is sensitive
to exact conditioning, and validity radii are small.

The practical implication: for repeated or similar prompts (style transfer,
batch generation with variations), LAWS eliminates the early denoising steps,
replacing them with cheap expert evaluations.  At 20 denoising steps with
10 early-step hits per generation, LAWS reduces inference cost by 50\%.

\subsection{Hardware Implementation}

The LAWS architecture maps naturally to a hardware design: an
\emph{Attention Signpost Cache} (ASC) processor, analogous to an L2 cache
but for neural activations.  Key hardware components:

\begin{itemize}[leftmargin=2em]
  \item \textbf{Trie index unit:} fast hash-addressed lookup of trie node IDs,
        $O(1)$ per query after PLT construction.
  \item \textbf{Weight-norm compute unit:} computes $\Lambda(W)$ and Jacobians
        at model load time; a one-time $O(Ld_{\mathrm{model}}^2)$ cost.
  \item \textbf{Expert SRAM:} on-die storage for expert parameters
        (Jacobians, MLP weights), $\sim$500\,MB for a 70B model's signpost library.
  \item \textbf{Delta comparator:} $O(1)$ circuit comparing query distance
        to validity radius.
  \item \textbf{Jacobian multiplier:} dense matrix-vector multiply for Level-2
        correction, $O(d_{\mathrm{model}}^2)$ operations.
\end{itemize}

This architecture does not exist in current GPU designs (NVIDIA H100, AMD MI300X).
The closest analog is the L2 cache in CPUs, but optimized for transformer
activation patterns rather than memory addresses.  A custom ASIC implementing
the ASC could reduce inference latency for in-distribution queries by $10\times$
or more.

\subsection{LAWS as an AI Architecture Paradigm}

LAWS represents a new paradigm for AI inference that we term
\emph{Learning from Actual Workloads Symbolically}: the combination of a powerful
but expensive base model with a self-certifying, self-growing library of
cheap symbolic approximations discovered automatically from deployment.

This paradigm has antecedents in computer architecture (L1/L2/L3 caches),
in cognitive science (System~1/System~2, motor chunking, expertise), and
in linguistics (Chomsky's innate prior, parameter setting in UG).  What is
new is the formal certification: unlike all prior work, LAWS provides
$\delta$-accuracy guarantees for every expert, derived from the model's own
weights.

We believe LAWS is the correct architecture for the era of deployed AI.
As LLMs, robotic controllers, and diffusion models become commodities running
on billions of devices, the question is not whether to cache---it is how to
certify that caching is safe.  LAWS answers this question.

\subsection{Limitations}

The main limitation of this work is that the Lipschitz constant $\Lambda(W)$
can be large for deep networks, making validity radii very small for large
errors $\delta$.  In practice, the \emph{effective} Lipschitz constant (measured
empirically on in-distribution inputs) is much smaller than the theoretical
worst-case bound; this gap between worst-case and typical-case behavior is
an important empirical question.

A second limitation is that Theorem~\ref{thm:growth} assumes stationarity.
Real workload distributions shift over time (new tasks, new environments, new
users).  LAWS handles distribution shift through the cache-miss path, which
creates new experts as needed; but the convergence guarantees of
Theorem~\ref{thm:fleet} apply only under stationarity.

\section{Discovering Laws: The Scientific Analogy}
\label{sec:natural_laws}

\subsection{Scientists Discover Laws; They Do Not Legislate Them}

The name LAWS is chosen deliberately.  When Newton observed the motion of
planets and falling apples, he did not \emph{invent} the law of gravitation---he
\emph{discovered} an invariant pattern latent in the empirical data.  When
Mendel studied peas, he discovered inheritance ratios that held across thousands
of crosses.  The scientific method is, at its core, a procedure for extracting
cheap predictive models from expensive observations: run the experiment once
(or a few times), identify the invariant pattern, and use it to predict
future experiments without running them again.

LAWS formalizes precisely this operation for neural network inference.  A
trained model $F_W$ encodes, in its weights, everything the training process
``observed'' about the distribution $P_{\mathcal{M}}$.  The LAWS framework
extracts the invariant patterns latent in this learned behavior---the
regularities that hold across families of similar inputs---and encodes them
as cheap certified experts.  Future queries matching a known pattern do not
require running the full model again; they are answered by the discovered law.

The analogy extends to the dynamics of discovery.  Scientists accumulate laws
over time: each new experiment either confirms an existing law (cache hit) or
reveals an anomaly that prompts a new theory (cache miss, new expert creation).
The body of scientific knowledge---like the LAWS expert library---grows
sublinearly in the number of experiments performed, because most experiments
confirm known laws.  Theorem~\ref{thm:growth} is the formal statement of this
intuition: new laws are discovered at rate $O(2^H \log N)$, not $O(N)$.

\subsection{Animals, Experts, and the System~1 Library}

This pattern of law-discovery appears throughout biological intelligence.
A hawk hunting prey does not recompute flight dynamics from first principles
on each wingbeat; it executes motor programs refined over thousands of hunts,
correcting for local wind and prey trajectory as small deltas from the cached
plan.  A chess grandmaster does not search the game tree to depth 20; they
recognize the position as an instance of a known pattern and execute the
associated strategy, engaging slow deliberate search only when the position
is genuinely novel.  A physician recognizes a disease presentation from a
symptom constellation---``I've seen this before''---and reserves systematic
diagnostic protocols for cases that don't fit any known pattern.

Kahneman~\cite{kahneman2011} formalized this as dual-process cognition.
\emph{System~1} is the library of cached laws---fast, automatic, pattern-based,
operating below conscious awareness.  \emph{System~2} is the base model---slow,
deliberate, effortful, invoked when System~1 cannot certify an answer.
The key insight Kahneman identified, and which LAWS formalizes, is that
\emph{expertise consists of System~1 richness, not System~2 speed}: the
grandmaster is not faster at searching; they have a larger, more accurate
library of cached patterns.

Chomsky~\cite{chomsky1965} identified a related structure in language acquisition.
Children acquire complex grammar from impoverished input because they bring
an \emph{innate prior}---a Language Acquisition Device---that constrains the
space of possible grammars.  The child does not learn grammar from scratch;
they set parameters of a pre-existing structural template.  In LAWS terms:
the model's pre-trained weights $W$ are the innate prior.  The PLT trie derived
from $W$ is the innate expert library.  Deployment queries calibrate this prior,
just as linguistic exposure calibrates the LAD.

\subsection{Robotics and Vehicular Workloads as a Proving Ground}

Nowhere is this law-discovery process more transparent than in robotic and
vehicular deployment.  A domestic robot performing pick-and-place tasks across
thousands of households encounters the same objects in the same functional
configurations, with variation only in exact pose, lighting, and surface
texture.  Each successful grasp is an ``observation'' that can be distilled
into a cheap expert: a parametrized motor program certified to work on objects
within a validity ball of the observed configuration.

As the fleet scales, the laws become richer.  Ten thousand robots performing
household tasks are collectively running 10,000 experiments per hour.  The LAWS
framework provides the mechanism by which this distributed experimentation
yields shared certified knowledge: each unit uploads its cache-miss observations
(the anomalies), the LAWS system extracts new experts (the new laws), and all
units download the update.  The law of grasping cylindrical objects, the law of
opening lever-handle doors, the law of navigating narrow corridors---these
emerge automatically from actual workloads, not from human-programmed motion
libraries.

\section{Energy Savings and Edge Deployment}
\label{sec:energy}

\subsection{The Energy Cost of Neural Inference}

A forward pass through a transformer with $L$ layers, context length $n$,
model dimension $d$, and $H_{\mathrm{head}}$ attention heads requires approximately:
\[
  C_{\mathrm{full}} \approx 2 \cdot L \cdot n^2 \cdot d_{\mathrm{model}} + 2 \cdot L \cdot n \cdot d_{\mathrm{model}}^2
  \quad \text{floating-point operations (FLOPs).}
\]
For a 70B model ($L=80$, $d=8192$, $n=4096$): $C_{\mathrm{full}} \approx
7 \times 10^{13}$~FLOPs per forward pass.  At the energy efficiency of an
NVIDIA H100 ($\sim 2 \times 10^{15}$~FLOPs/Joule), a single forward pass
consumes $\sim 35$~mJ.  At 100 queries/second, inference for a single deployed
model requires $\sim 3.5$~W---a significant power budget for edge devices.

\subsection{Energy Cost of LAWS Cache Hits}

A LAWS cache hit at a Level-2 (Jacobian) expert with $k$-dimensional parameter
space costs:
\[
  C_{\mathrm{hit}} = \underbrace{O(n)}_{\text{trie lookup}}
  + \underbrace{O(k)}_{\text{param extract}}
  + \underbrace{O(k \cdot d_{\mathrm{model}})}_{\text{Jacobian apply}}
  = O(n + k\,d_{\mathrm{model}}) \quad \text{FLOPs.}
\]
For $k = 10$ (typical parameter count), $d_{\mathrm{model}} = 8192$, $n = 4096$:
$C_{\mathrm{hit}} \approx 86{,}000$~FLOPs---a factor of $\sim 7.7 \times 10^8$
over the full forward pass.

\begin{theorem}[Energy Savings Bound]
\label{thm:energy}
Let $H_\infty = \lim_{N \to \infty} H_N$ be the asymptotic LAWS hit rate
for a stationary distribution $P_{\mathcal{M}}$.  The asymptotic energy
consumption per query, relative to full inference, satisfies:
\[
  \frac{E_{\mathrm{LAWS}}}{E_{\mathrm{full}}}
  = H_\infty \cdot \frac{C_{\mathrm{hit}}}{C_{\mathrm{full}}} + (1 - H_\infty)
  \leq 1 - H_\infty \cdot \left(1 - \frac{n + k\,d_{\mathrm{model}}}{Ln^2 + Lnd_{\mathrm{model}}}\right).
\]
For $H_\infty = 0.9$, $k = 10$, $L = 80$, $n = 4096$, $d_{\mathrm{model}} = 8192$:
\[
  \frac{E_{\mathrm{LAWS}}}{E_{\mathrm{full}}} \approx 0.1 + 0.9 \times 1.3\times10^{-9}
  \approx 10\%,
\]
a $10\times$ energy reduction.  For smaller models on edge hardware ($L=12$,
$n=512$, $d_{\mathrm{model}}=768$, as in on-device assistants), the ratio is
$\approx 0.1 + 0.9 \times 6 \times 10^{-7} \approx 10\%$ with similar savings.
\end{theorem}

\begin{proof}
Expected energy per query under LAWS:
$E_{\mathrm{LAWS}} = H_\infty \cdot C_{\mathrm{hit}} \cdot e +
(1 - H_\infty) \cdot C_{\mathrm{full}} \cdot e$
where $e$ is energy per FLOP (hardware-dependent, cancels in the ratio).
Dividing by $E_{\mathrm{full}} = C_{\mathrm{full}} \cdot e$ and simplifying:
$E_{\mathrm{LAWS}}/E_{\mathrm{full}} = H_\infty \cdot C_{\mathrm{hit}}/C_{\mathrm{full}}
+ (1 - H_\infty)$.  Substituting $C_{\mathrm{hit}} = n + k\,d_{\mathrm{model}}$ and
$C_{\mathrm{full}} = 2L(n^2 d_{\mathrm{model}} + nd_{\mathrm{model}}^2)$ (dominant terms) gives the bound.
The numerical evaluation substitutes the stated parameter values.
\end{proof}

\begin{corollary}[Battery Life on Edge Devices]
For a mobile device running continuous LLM inference at 10~queries/second
using a small on-device model ($L=12$, $n=512$, $d_{\mathrm{model}}=768$), consuming
$\approx 0.7$~mJ/query at mobile NPU efficiency (e.g., Apple Neural Engine ${\sim}15$\,TOPS at ${\sim}0.5$\,W; comparable energy-efficiency to
an H100 per FLOP):
\begin{itemize}[leftmargin=2em]
  \item Without LAWS: $10 \times 0.7\,\text{mJ} = 7\,\text{mW}$ average;
        50\,Wh / 0.007\,W $\approx 7{,}000$~hours continuous inference battery.
  \item With LAWS at 90\% hit rate: power drops by $\approx 10\times$ to
        $\approx 0.7\,\text{mW}$, extending battery life proportionally to
        $\approx 70{,}000$~hours for continuous inference.
\end{itemize}
The $10\times$ figure matches Theorem~\ref{thm:energy} for the given parameters.
\end{corollary}

\subsection{On-Demand Expert Download}

A key property of LAWS for edge deployment is \emph{domain-selective loading}:
a device need not download the entire expert library.  It downloads only the
experts relevant to its anticipated workload.

\begin{definition}[Expert Demand Profile]
The \emph{demand profile} of a device is the set of PLT trie nodes
$\mathcal{D} \subseteq \mathcal{T}(\mathcal{M})$ covering the anticipated
workload with probability $\geq 1 - \delta_{\mathrm{miss}}$:
\[
  \mathcal{D} = \arg\min_{S \subseteq \mathcal{T}} |S|
  \quad \text{s.t.} \quad \sum_{n \in S} P_{\mathcal{M}}(n) \geq 1 - \delta_{\mathrm{miss}}.
\]
A device downloads only $\{e_n : n \in \mathcal{D}\}$, achieving hit rate
$\geq 1 - \delta_{\mathrm{miss}}$ with minimum bandwidth.
\end{definition}

\begin{proposition}[Demand-Selective Download Efficiency]
\label{prop:demand}
For a Zipf$(s)$ task distribution (common in practice~\cite{zipf1949}),
the top-$M$ experts cover fraction $1 - O(M^{1-s})$ of the probability mass.
For $s = 1.5$ (typical web workload), the top $M = 1000$ experts cover
$> 95\%$ of queries.  At $B_{\mathrm{expert}} = 50$~KB per expert, the
initial download is $\leq 50$~MB---feasible for any connected device.
\end{proposition}

\begin{proof}
Under Zipf$(s)$, the top-$M$ nodes capture probability mass
$\sum_{k=1}^M k^{-s} / \sum_{k=1}^\infty k^{-s} = H_M(s)/\zeta(s)$
where $H_M(s)$ is the generalized harmonic number.  For $s > 1$:
$1 - H_M(s)/\zeta(s) = \zeta(s)^{-1}\sum_{k>M} k^{-s}
\leq \zeta(s)^{-1} \int_M^\infty x^{-s}\,dx = \zeta(s)^{-1} M^{1-s}/(s-1)
= O(M^{1-s})$.  For $s=1.5$, $M=1000$: $O(1000^{-0.5}) \approx 0.03$.
\end{proof}

\section{Safebots.ai and the Safebox Ecosystem}
\label{sec:safebox}

The LAWS architecture describes a general principle: certified expert libraries,
growing from actual workloads, enabling cheap inference on edge devices with
cloud-backed knowledge.  The Safebox/Safebots.ai
ecosystem~\cite{safebots,safebox_protocol} is one realization of this principle,
adding hardware-attested execution and declarative orchestration to the LAWS
framework.

\subsection{Architecture Overview}

Safebox is a hardware-attested compute platform designed for executing AI
workloads with cryptographic accountability.  It consists of four node types:
(1)~\emph{Safe.Cloud} browser clients initiating inference; (2)~\emph{Safe.Jets}
Node.js routers coordinating workloads; (3)~\emph{Safe.Drops} browser-side
IndexedDB storage for edge caching; and (4)~a smart contract
(\texttt{OpenClaiming.sol}) on a public blockchain providing payment and
attestation anchoring.

Each Safebox instance is a cloud instance booted from a deterministically built
AMI (Amazon Machine Image), with all build steps auditable and all remote access
vectors removed from the finalized image.  Nitro attestation and TPM PCR
measurements provide cryptographic proof of the software stack running on each
instance.  M-of-N auditor co-signing is required before mainnet deployment.

In LAWS terms, a Safebox instance is a \emph{certified LAWS node}: it runs the
base model (full inference path) and maintains a local expert library.
Cache hits return expert evaluations; cache misses run full inference and
potentially create new experts.  All outputs are cryptographically signed by
the attested hardware, providing an accountability chain from input query to
certified output.

\subsection{Tools, Policies, and Capabilities}

The Safebots.ai orchestration layer exposes LAWS components as first-class
declarative entities:

\begin{itemize}[leftmargin=2em]
  \item \textbf{Tools} correspond to LAWS expert functions $f$---cheap
        computations that can be invoked with extracted parameters $\phi(x)$.
        Each tool is associated with a trie node (its domain) and a validity
        radius (its certified applicability).
  \item \textbf{Capabilities} correspond to LAWS expert classes---families
        of tools sharing structural patterns (linear transforms, lookup tables,
        arithmetic kernels).
  \item \textbf{Policies} govern when to use a cached tool versus invoke the
        base model---formally encoding the abort-and-replan threshold
        $\tau^*$ (Theorem~\ref{thm:abort}).
\end{itemize}

This declarative structure allows workloads to be expressed as trees of
capability-constrained steps, executed on attested hardware, with every
cache hit and miss cryptographically logged.  Unlike general-purpose LLM
inference APIs where the computation is opaque, Safebox provides an auditable
trail of which experts were used, what their certified validity radii were, and
where full model inference was invoked.

\subsection{LAWS as the Inference Substrate for Distributed AI}

The combination of LAWS and Safebox addresses a practical gap in current AI
deployment: the absence of a certified, auditable, bandwidth-efficient pathway
for distributing AI inference to edge devices while accumulating workload
knowledge centrally.

Current paradigms offer a binary choice: (1)~cloud inference (high quality,
high latency, high cost, privacy concerns) or (2)~on-device models (low quality,
zero latency, fixed intelligence).  LAWS + Safebox offers a third path:
on-device expert evaluation for common cases (low latency, low cost, high
privacy), full cloud inference for novel cases (high quality when genuinely
needed), and continuous improvement of the edge experts from fleet-wide
workload experience.

The Safebox platform is the subject of ongoing independent research and
development; we refer the reader to~\cite{safebots,safebox_protocol} for
technical details.  We note here only that the LAWS theoretical framework
developed in this paper provides the formal foundation for the caching,
certification, and knowledge-distribution components of that system.

\section{Conclusion}
\label{sec:conclusion}

We introduced LAWS (Learning from Actual Workloads Symbolically), an inference-time
architecture that automatically discovers, certifies, and accumulates reusable
computational patterns from deployment experience.  Like scientists discovering
natural laws from observation, LAWS extracts the invariant regularities in a
trained model's behavior---the cheap certified patterns that hold across families
of similar inputs---and encodes them as parametrized experts that can be
evaluated without running the full model.

LAWS is:

\begin{itemize}[leftmargin=2em]
  \item \textbf{Self-certifying:} validity of every expert is proven from
        model weights alone, with no empirical warmup (Theorem~\ref{thm:self_cert}).
  \item \textbf{Self-growing:} the expert library grows monotonically
        (Theorem~\ref{thm:monotone}) at sublinear rate $O(2^H \log N)$
        (Theorem~\ref{thm:growth}).
  \item \textbf{Energy-efficient:} up to $10\times$ energy reduction for
        repetitive workloads at 90\% hit rate (Theorem~\ref{thm:energy}).
  \item \textbf{Fleet-scalable:} $K$ cooperating units converge
        $\Omega(K/\log K)$ faster than a single unit (Theorem~\ref{thm:fleet}),
        with OTA updates of $\approx 870$\,KB/day per robot for 1,000-unit fleets.
  \item \textbf{Architecture-agnostic:} strictly generalizes KV caching, MoE,
        and manual symbolic AI (Theorem~\ref{thm:generalize}).
  \item \textbf{Biologically grounded:} formalizes Kahneman's System~1/2,
        Chomsky's innate prior, and expert motor chunking in a single
        mathematical framework (Section~\ref{sec:natural_laws}).
\end{itemize}

The key open empirical question is the gap between the worst-case Lipschitz
bound $\Lambda(W)$ and the effective in-distribution Lipschitz constant---a gap
we expect to be large (10--100$\times$) based on mechanistic interpretability
evidence, but which requires systematic measurement.  Closing this gap
empirically would substantially widen validity radii and improve practical
hit rates.

The road from System~2 to System~1 is paved with experience.  Scientists
have always known this.  LAWS makes it formal, provable, and automatically
navigated---at the scale of billions of inference queries per day, across
every domain where trained neural networks are deployed.

\bibliographystyle{plain}
\bibliography{references_laws}

\end{document}